\documentclass[10pt,journal,epsfig]{IEEEtran}

\usepackage[dvips]{graphicx}
\usepackage{graphicx}
\usepackage{amssymb}
\usepackage{cite}
\usepackage{subfigure}
\usepackage{amsmath}
\usepackage{algorithm}
\usepackage{algorithmic}
\usepackage{multirow}
\usepackage{xcolor}
\usepackage{epstopdf}
\usepackage{amsthm}
\allowdisplaybreaks[4]

\begin{document}

\title{Communication Efficient ConFederated Learning: An Event-Triggered SAGA Approach}

\author{Bin Wang, Jun Fang, Hongbin Li, ~\IEEEmembership{Fellow,~IEEE},
Yonina C. Eldar, ~\IEEEmembership{Fellow,~IEEE}
\thanks{Part of this work was accepted by ICASSP 2024. This
paper has been accepted by IEEE Transactions on Signal Processing.}
\thanks{Bin Wang and Jun Fang are with the National Key Laboratory
of Wireless Communications, University of
Electronic Science and Technology of China, Chengdu 611731, China,
Email: JunFang@uestc.edu.cn}
\thanks{Hongbin Li is with the Department of Electrical and Computer Engineering,
Stevens Institute of Technology, Hoboken, NJ 07030, USA, E-mail:
Hongbin.Li@stevens.edu}
\thanks{Yonina C. Eldar is with the Faculty of Mathematics and Computer Science,
Weizmann Institute of Science, Rehovot 7610001, Israel, E-mail:
yonina.eldar@weizmann.ac.il}
\thanks{The work of J. Fang was supported in part by the Sichuan Science and Technology Program
under Grant 2023ZYD0146, and in part by the National Key
Laboratory of Wireless Communications Foundation. The work of H.
Li was supported in part by the National Science Foundation under
Grants ECCS-1923739, ECCS-2212940, and CCF-2316865.} }

\maketitle

\begin{abstract}
Federated learning (FL) is a machine learning paradigm that
targets model training without gathering the local data dispersed
over various data sources. Standard FL, which employs a single
server, can only support a limited number of users, leading
to degraded learning capability. In this work, we consider a
multi-server FL framework, referred to as \emph{Confederated
Learning} (CFL), in order to accommodate a larger number of
users. A CFL system is composed of multiple networked edge
servers, with each server connected to an individual set of users.
Decentralized collaboration among servers is leveraged to harness
all users' data for model training. Due to the potentially massive
number of users involved, it is crucial to reduce the communication
overhead of the CFL system. We propose a stochastic
gradient method for distributed learning in the CFL framework.
The proposed method incorporates a conditionally-triggered user
selection (CTUS) mechanism as the central component to effectively
reduce communication overhead. Relying on a delicately designed
triggering condition, the CTUS mechanism allows each server to
select only a small number of users to upload their gradients,
without significantly jeopardizing the convergence performance of
the algorithm. Our theoretical analysis reveals that the proposed
algorithm enjoys a linear convergence rate. Simulation results
show that it achieves substantial improvement over state-of-the-art
algorithms in terms of communication efficiency.
\end{abstract}


\section{Introduction}
\label{sec-1}
The tremendous advancement of machine learning (ML) has rendered it
a driving force for various research fields and industrial applications.
However, the traditional ML framework follows a centralized fashion
which assembles the training data to a central computing unit (CPU)
where model training is performed. Such an approach might be problematic
when data is confidential or when transferring the training data
to the CPU is unrealistic. With a growing interest in data privacy,
regulations like GDPR (General Data Protection Regulation) and ADPPA
(American Data Privacy and Protection Act) have imposed restrictions
on sharing privacy-sensitive data among different clients or platforms.
As such, breaking the data-privacy barrier is an urgent and meaningful
task.

Federated Learning (FL) \cite{KonecnyMcMahan16,LiSahu20} is an
emerging machine learning paradigm that enables model training
without transferring local data to the CPU. FL has drawn
significant attention from both academia and industry,
especially for privacy-sensitive and data-intensive applications.
A standard FL system consists of a server and a set of
devices/users. In general, FL addresses privacy protection by
adopting a compute-then-aggregate (CTA) approach. More precisely,
in each iteration the server first broadcasts the global model
vector to the users. Each user then computes a local gradient
using its own data, and uploads its local gradient to the
server. At the end of each iteration, the server performs one
step of gradient descent (using the aggregated gradient) to obtain
an updated global model vector. This process cycles until
model training is accomplished. Typically, the training process
takes a large number of iterations to converge. Thus FL may
consume a substantial amount of communication resources.
Therefore, it is important to reduce the communication
overhead to an affordable level. To this end, various methods were
developed along different research lines, including methods which
aim at improving the convergence speed
\cite{Stich18,HaddadpourKamani19,YuanMa20,LiKovalev20,CondatAgarsky22,
MishchenkoMalinovsky22,WangJoshi21,PathakWainwright20,CenZhang20,
ZhangHong21,LiHuang19,SahuLi18,LiuChen22}, methods that reduce the
amount of transmission by selecting only a subset of users for
uploading their gradients \cite{YangLiu19,AmiriGunduz21,RenHe20,
ReisizadehMokhtari20,ChenShlezinger21,DinhPham21,ChenGiannakis18,
ChenSun21,SunChen22}, methods that sparsify or quantize the local
gradients \cite{AjiHeafield17,AlistarhGrubic17,BernsteinWang18,
KarimireddyRebjock19,WuHuang18,ShlezingerChen20,StichCordonnier18,
BeznosikovHorvath20,HorvathKovalev22,RichtarikSokolov22}, or
combinations of these techniques.

Besides the excessively high communication cost, another
restriction of FL is that conventional FL systems employ only
a single server. Due to the limited communication capacity, the
number of users that can be served by a single server is limited.
To involve more devices for model training, an alternative
framework to the standard single-server FL is a decentralized
FL system
\cite{HegedusDanner19,SavazziNicoli20,XingSimeone21,KoloskovaStich19,
YeLiang22,LiuChen22,XinKhan20,XinKhan22,KovalevKoloskova21,SinghData22}.
A decentralized FL system is composed of a number of nodes or
agents which are able to perform computation and communication.
Each node carries its own training data. Different nodes form a
decentralized network in which only neighboring nodes can
either bidirectionally or directionally
\cite{QureshiXin21,QureshiXin22,QureshiKhan22} communicate with
each other. In decentralized FL, the
training process follows a similar CTA mode as in the
standard system, except that the local gradient or local model
vector is exchanged among neighboring nodes. Despite its
scalability, decentralized FL is confined to D2D
(device-to-device) type networks which requires D2D
communications that may not be easily achieved in cellular
systems.

Recently, a new FL framework termed \emph{Confederated Learning}
(CFL) was proposed in \cite{WangFang23} to overcome the drawbacks
of existing FL systems. A CFL system consists of multiple servers,
in which each server is connected with an individual set of users
as in the conventional FL framework. Decentralized collaboration
among servers is leveraged to make full use of the data dispersed
over different users. CFL can be considered as a hybrid of
standard and decentralized FL systems. In particular, CFL
degenerates to standard FL when there is only a single server.
Although there exist a plethora of algorithms/convergence analyses
for standard FL, the extension of these results to CFL is not
straightforward since the latter framework involves decentralized
collaboration among servers. On the other hand, CFL becomes a
decentralized FL system when there is no user, namely, when each
server itself carries the training data. In this case, each
server's data are readily accessible to this server without any
communication cost. This is in sharp contrast to the CFL framework
whose communication cost mainly comes from collecting training
information by each server from its associated users. Therefore,
existing gradient tracking-based decentralized optimization methods
\cite{XinKhan20,XinKhan22,QureshiXin21,QureshiXin22,QureshiKhan22},
when applied to CFL, lead to an unsatisfactory communication
efficiency. In \cite{WangFang23}, a stochastic ADMM algorithm
with random user selection is developed for CFL. However, the
ADMM-based method is proved to possess only a sub-linear
convergence rate, and its performance relies heavily on man-crafted
parameters that can be hard to tune in real-world applications.

In this paper, we propose a gradient-based method for
communication-efficient CFL. The proposed algorithm is based on
the framework of GT-SAGA (gradient tracking with stochastic
average gradient) \cite{XinKhan20}. To reduce the amount of data
transmission between servers and users, a conditionally-triggered
user selection (CTUS) mechanism is developed. CTUS sets a
computationally verifiable selection criterion at the user side
such that only those users whose VR-SGs (variance-reduced
stochastic gradients) are sufficiently informative report their
VR-SGs to their associated servers. At the server side, the
aggregated gradient is obtained by integrating the uploaded VR-SGs
as well as the stale VR-SGs corresponding to those unreported
users. The CTUS mechanism shares a similar spirit with the
even-triggering-based methods proposed for standard FL or
traditional decentralized optimization methods
\cite{KiaCortes15,KajiyamaHayashi18,SinghData22,GeorgeGurram20,
GaoDeng21,ZehtabiHosseinalipour22,RichtarikSokolov22,
ChenGiannakis18,ChenSun21,ChenBlum22,SunChen22}. Nevertheless, the
selection criterion developed in this paper is very different from
existing methods. Specifically, for multi-server systems, the
variables from neighboring servers should be taken into account in
the design of the selection criterion. Simulation results show
that the proposed CTUS mechanism helps preclude most of those
non-informative user uploads, thereby striking a higher
communication efficiency than state-of-the-art algorithms.

The rest of this paper is organized as follows. In Section
\ref{sec-2}, we introduce the confederated learning problem along
with some assumptions on the objective function as well as the
server network. Then, in Section \ref{sec-3}, we provide a brief
overview of the classic gradient tracking (GT) method as well as
the GT-SAGA method that can be adapted to solve the CFL problem.
The proposed method is presented in Section \ref{sec-4}, with the
convergence analysis given in Section \ref{sec-5}. The proof of
the main theoretical result, namely, Theorem 1, is provided in
Section \ref{sec-6}. In section \ref{sec-7}, we also provide
theoretical analysis to justify that the proposed CTUS can save
user uploads under mild conditions. Simulations results are
presented in Section \ref{sec-8}, followed by concluding remarks
in Section \ref{sec-9}.

\section{Problem Formulation} \label{sec-2}
\subsection{CFL Framework}
 We consider a confederated learning (CFL) framework
consisting of $N$ networked edge servers. Figure \ref{fig1-1}
depicts a schematic of CFL. The connective relation of these edge
servers is described by an undirected connected graph $G=\{V,E\}$,
where $V$ (resp. $E$) denotes the set of servers (resp. edges).
The $i$th edge server serves $P_i$ users. Each user is only
allowed to communicate with its associated server. In addition to
communicating with its own users, each server can communicate with
its neighboring servers. With the confederated network, we aim to
solve the following CFL problem:
\begin{align}
\mathop {\min }\limits_{\boldsymbol{x}\in\mathbb{R}^d}
& \ \textstyle f(\boldsymbol{x})\triangleq
\frac{1}{N}\sum_{i=1}^{N} f_i(\boldsymbol{x}),
\label{problem-cfl}
\end{align}
where $\boldsymbol{x}\in\mathbb{R}^{d}$ is the model vector to be
learned, $f_{i}(\boldsymbol{x})=\sum_{j=1}^{P_i}
f_{ij}(\boldsymbol{x})$, $f_{ij}(\boldsymbol{x})=
\sum_{t=1}^{S_{ij}} f_{ij,t}(\boldsymbol{x})$
is the loss function held by user $u_{ij}$,
$f_{ij,t}(\boldsymbol{x})$ is the loss function
corresponding to the $t$th training sample at user $u_{ij}$,
and $S_{ij}$ is the number of training samples at user
$u_{ij}$. Here user $u_{ij}$ refers to the $j$th user served by
the $i$th server. It is also noteworthy that $f_{ij,t}$
may corresponds to a mini-batch of training samples instead of a
single sample.

The communication bottleneck of CFL lies in the user-to-server
(U2S) communications. Existing methods are designed either for
standard single-server FL or for decentralized FL. Standard FL
methods cannot be straightforwardly extended to the CFL, while
decentralized FL methods neglect the U2S communications in their
algorithmic development. Focusing on problem (\ref{problem-cfl}),
we aim to develop a communication-efficient method which seeks to
reduce the U2S communication overhead.

\begin{figure}
\centering
\includegraphics[width=7cm,height=5cm]{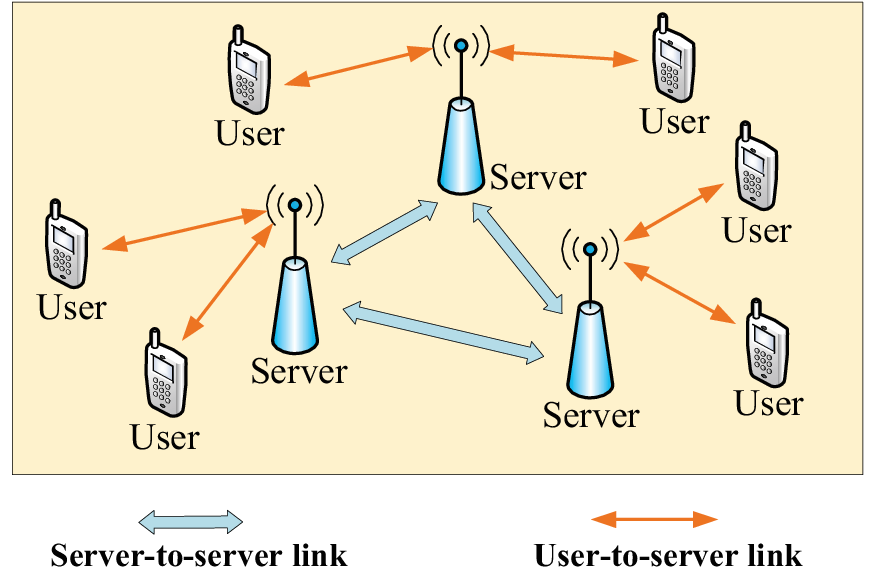}
\caption{The CFL framework with multiple servers.}
\label{fig1-1}
\end{figure}

\subsection{Function and Server-Network Assumptions}
\label{sec-obj-assumption}
We assume that $f$ (resp. $\nabla f$) is $\mu$-strongly convex
(resp. $L$-Lipschitz continuous) while both $f$, $f_i$, $f_{ij}$
and $f_{ij,t}$ are continuously differentiable with their gradients
being $L$-Lipschitz continuous. The definitions of $\mu$-strongly
convexity and $L$-Lipschitz continuity are given below.

\newtheorem{definition}{Definition}
\begin{definition}
(Strongly convexity) A function $f:\mathbb{R}^d\rightarrow
\mathbb{R}\cup\{+\infty\}$ is said to be $\mu$-strongly convex if
\begin{align}
\textstyle f(\boldsymbol{y})\geq f(\boldsymbol{x})+\langle\nabla
f(\boldsymbol{x}),\boldsymbol{y}-\boldsymbol{x}\rangle+
\frac{\mu}{2}\|\boldsymbol{y}-\boldsymbol{x}\|_2^2, \ \forall
\boldsymbol{x}, \boldsymbol{y}.
\label{assumption-strong-1}
\end{align}
\end{definition}

\newtheorem{definition2}{Definition}
\begin{definition}
\label{assumpt-lip} (Lipschitz continuity) The gradient of
$f:\mathbb{R}^d \rightarrow \mathbb{R}\cup\{+\infty\}$ is said to
be $L$-Lipschitz continuous if
\begin{align}
\|\nabla f(\boldsymbol{x})-\nabla f(\boldsymbol{y})\|_2\leq L
\|\boldsymbol{x}-\boldsymbol{y}\|_2, \ \forall \boldsymbol{x},
\boldsymbol{y}.
\label{def-lip2}
\end{align}
\end{definition}

Recall that the servers form a bidirectionally connected graph
$G=\{V,E\}$. Denote $\boldsymbol{W}\in\mathbb{R}_{+}^{N\times N}$
as the mixing matrix associated with the graph $G$. It is assumed
that $\boldsymbol{W}$ is symmetric, primitive and doubly stochastic.
In particular, $w_{ii'}$ which is the $(i,i')$th element of
$\boldsymbol{W}$ equals to $0$ (resp. nonzero) if server $i$ and
$i'$ are unconnected (resp. connected). For such a $\boldsymbol{W}$,
its largest singular value is $1$ (with multiplicity equals to $1$),
with its corresponding singular vector being
$\frac{1}{\sqrt{N}}\boldsymbol{1}_{N}$. The second largest singular
value of $\boldsymbol{W}$, denoted as $\sigma$, is thus equal to
$\|\boldsymbol{W}-\frac{1}{N}\boldsymbol{1}_{N}\boldsymbol{1}_{N}^T\|_2$
and is smaller than $1$. Notably, $\boldsymbol{W}$ can be
conveniently obtained as $\boldsymbol{W}=\boldsymbol{I}-
\frac{\boldsymbol{L}}{\tau}$, where $\boldsymbol{L}$ is the
Laplacian matrix of $G$ and $\tau>\frac{1}{2}
\lambda_{\text{max}}(\boldsymbol{L})$ is a scaling factor.

\section{Overview of GT and GT-SAGA}
\label{sec-3} Our proposed algorithm is based on the gradient
tracking (GT) framework \cite{LorenzoScutari16}. In this section,
we begin with a brief introduction of GT and then introduce the
GT-SAGA algorithm \cite{XinKhan20} which is a practical variant of
GT. GT-SAGA is designed for solving decentralized optimization
problems. We will discuss how to adapt GT-SAGA to solve the CFL
problem (\ref{problem-cfl}) in Section \ref{sec-3-c}.

\subsection{Gradient Tracking}
GT \cite{LorenzoScutari16} is designed for solving decentralized
optimization problems of the following form:
\begin{align}
\textstyle\mathop {\min }
\limits_{\boldsymbol{x}\in\mathbb{R}^d} \ F(\boldsymbol{x})
\triangleq \frac{1}{N}\sum_{i=1}^{N} f_i(\boldsymbol{x}),
\label{problem-gt}
\end{align}
where $f_i(\boldsymbol{x})=\sum_{t=1}^{P_i}f_{i,t}
(\boldsymbol{x};\mathcal{D}_{i,t})$ is the local loss
function held by node $i$ and $f_{i,t}(\boldsymbol{x})$ is the
loss function associated with the $t$th training samples stored
at node $i$. GT is usually compactly written as
\begin{align}
&\text{GT}: \ \textstyle \boldsymbol{x}_i^{k+1}
=\sum_{i'=1}^N w_{ii'}\boldsymbol{x}_{i'}^k-\alpha\boldsymbol{y}_i^k,
\ 1\leq i\leq N,
\label{GT-alg-x}\\
&\textstyle\boldsymbol{y}_i^{k+1}=\sum_{i'=1}^N w_{ii'}\boldsymbol{y}_{i'}^{k}
+\nabla f_i(\boldsymbol{x}_i^{k+1})-\nabla f_i(\boldsymbol{x}_i^{k}),
\ 1\leq i\leq N.
\label{GT-alg-y}
\end{align}

Note that GT is designed for a D2D network in which each node
carries its own training data and each node can communicate
with its neighboring nodes. In GT, each node holds two variables,
$\boldsymbol{x}_i$ and $\boldsymbol{y}_i$. After the $(k+1)$th
iteration, each node exchanges $(\boldsymbol{x}_i^{k+1},
\boldsymbol{y}_i^{k+1})$ with its neighboring nodes. The core
idea behind GT is the combination of decentralized gradient
descent (DGD) and dynamic average consensus (DAC). To see this,
omitting the DAC step, i.e., (\ref{GT-alg-y}), and assuming that
$\boldsymbol{y}_i^k =\nabla f_i(\boldsymbol{x}_i^k)$, then GT
degenerates to the standard DGD algorithm, in which $\alpha$
is the stepsize. However, it is well known that the
exact convergence of DGD can not be guaranteed unless a decreasing
stepsize is employed. The problem is that a decreasing stepsize
can only offer a sublinear convergence rate even if $f_i$ is
strongly convex. In GT, the DAC mechanism is incorporated to
remedy this drawback of DGD. DAC is an efficient tool to track
the average of time-varying signals. Formally, suppose each
node $i$ measures a time-varying signal $r_i^k$ at time $k$
and consider the problem of tracking its average $\bar{r}^k=
\frac{1}{N}\sum_{i=1}^N r_i^k$ at each node. The DAC mechanism,
which is mathematically stated as
\begin{align}
\textstyle d_i^{k+1}=\sum_{i'=1}^N w_{ii'}d_{i'}^{k}+r_i^{k+1}-r_i^k,
\ \forall i,
\end{align}
converges to $\bar{r}^{k+1}$ provided that $\lim_{k\rightarrow\infty}
|r_i^{k+1}-r_i^k|= 0$. In GT, we intend to track the average
of the local gradients $\frac{1}{N}\sum_{i=1}^N
\nabla f_i(\boldsymbol{x}_i^k)$ instead of using only the local
gradient $\nabla f_i(\boldsymbol{x}_i^k)$ at every node. This
generates the DAC step (\ref{GT-alg-y}). If the local variables
tend to arrive at a consensus state, i.e. $\boldsymbol{x}_i^k
\rightarrow\boldsymbol{x}$, which also means that
$\nabla f_i(\boldsymbol{x}_i^{k+1})-\nabla f_i(\boldsymbol{x}_i^{k})
\rightarrow 0$, then (\ref{GT-alg-y}) ensures that $\boldsymbol{y}_i^{k}
\rightarrow \frac{1}{N}\sum_{i=1}^N \nabla f_i(\boldsymbol{x})$ and
thus (\ref{GT-alg-x}) degenerates to a gradient descent step
applied to the whole objective function $F$. As such, GT is
guaranteed to converge to the global optimum with a linear
convergence rate, under the strongly convexity assumption.

\subsection{Gradient Tracking with Variance Reduction}
In machine learning applications, each user may hold a large
number of training samples and thus it is neither practical nor
efficient to compute the full local gradient $\nabla
f_i(\boldsymbol{x}_i^{k+1})$. An alternative solution is to
compute a stochastic approximation of
$\nabla f_i(\boldsymbol{x}_i^{k+1})$. However, directly employing
the stochastic gradient introduces a non-vanishing variance
and would consequently undermine the exact convergence of GT.
To alleviate this problem, GT-SAGA \cite{XinKhan20}, summarized
in Algorithm \ref{alg-2}, was proposed to incorporate a
variance-reduced stochastic gradient (VR-SG) $\boldsymbol{g}_i^{k+1}$
to replace $\nabla f_i(\boldsymbol{x}_i^{k+1})$. The VR-SG is
an unbiased estimate of $\nabla f_i(\boldsymbol{x}_i^{k+1})$ in
the sense that $\mathbb{E}_{t_i^{k+1}}\{\boldsymbol{g}_i^{k+1}\}=
\nabla f_i(\boldsymbol{x}_i^{k+1})$. More importantly, the variance
of $\boldsymbol{g}_i^{k+1}$ which is mathematically stated as
$\mathbb{E}_{t_i^{k+1}}\{\|\boldsymbol{g}_i^{k+1}- \nabla
f_i(\boldsymbol{x}_i^{k+1})\|_2^2\}$ tends to $0$ if the
algorithm converges. Thanks to the VR technique, GT-SAGA is
guaranteed to converge to the global optimum while still
maintaining a linear convergence rate.

In addition to GT-SAGA, there also exist other variance
reduction-based gradient tracking methods, e.g.
\cite{XinKhan20,XinKhan22,QureshiXin21,QureshiXin22,QureshiKhan22}.
Among them, Push-SAGA/AB-SAGA \cite{QureshiXin21,QureshiXin22} and
Push-SVRG/AB-SVRG \cite{QureshiKhan22} are designed for directed
networks, GT-SVRG \cite{XinKhan20} and GT-SARAH \cite{XinKhan22}
are based on double-loop variance reduction techniques that
periodically demand all users to upload their local gradients to
their respective servers. Such a requirement poses practical
challenges in the FL setting.

\begin{algorithm}
\caption{GT-SAGA}
\label{alg-2}
\begin{algorithmic}
\STATE{\textbf{Input}: $N$, $\{P_i\}$, $\boldsymbol{W}$, $\alpha$,
$\{\boldsymbol{x}_i^{0},\boldsymbol{y}_i^{0},\boldsymbol{\phi}_{i,t}^{0}=
\boldsymbol{0}\}_{i=1,t=1}^{i=N,t=P_i}$.} \STATE{\textbf{while}
not converge \textbf{do}} \STATE{\quad\textbf{For} each node
\textbf{parallel do}} \noindent \STATE{ \quad\quad\textbf{1.} Node
$i$ computes $\boldsymbol{x}_i^{k+1}=\sum_{i'=1}^N w_{ii'}
\boldsymbol{x}_{i'}^k-\alpha\boldsymbol{y}_i^k$;} \noindent
\STATE{ \quad\quad\textbf{2.} Node $i$ uniformly generates a
random integer $t_{i}^{k+1}$, $1\leq t_{i}^{k+1}\leq P_{i}$, and
then computes the VR-SG via
\begin{align}
\textstyle\boldsymbol{g}_i^{k+1}=&P_i\cdot(\nabla f_{i,t_i^{k+1}}(\boldsymbol{x}_i^{k+1})-
\nabla f_{i,t_i^{k+1}}(\boldsymbol{\phi}_{i,t_i^{k+1}}^{k}))+
\nonumber\\
&\textstyle
\sum_{t=1}^{P_i}\nabla f_{i,t}(\boldsymbol{\phi}_{i,t}^{k}),
\end{align}
Set $\nabla f_{i,t_i^{k+1}}(\boldsymbol{\phi}_{i,t_i^{k+1}}^{k+1})
=\nabla f_{i,t_i^{k+1}}(\boldsymbol{x}_i^{k+1})$ and also set
$\nabla f_{i,t}(\boldsymbol{\phi}_{i,t}^{k+1})=
\nabla f_{i,t}(\boldsymbol{\phi}_{i,t}^{k})$,
$\forall t\neq t_i^{k+1}$;}
\noindent \STATE{ \quad\quad\textbf{3.} Node $i$ computes
$\boldsymbol{y}_i^{k+1}=\sum_{i'=1}^N w_{ii'}
\boldsymbol{y}_{i'}^{k}+\boldsymbol{g}_i^{k+1}-\boldsymbol{g}_i^k$
and then broadcasts ($\boldsymbol{x}_i^{k+1}$,$\boldsymbol{y}_i^{k+1}$)
to its neighboring nodes;}
\STATE{\quad\textbf{End For}}
\par{\textbf{End while} and \textbf{Output $\boldsymbol{x}^{k+1}$};}
\end{algorithmic}
\end{algorithm}

\subsection{Adapting GT-SAGA for CFL}
\label{sec-3-c} Now we discuss how to modify GT-SAGA to make it
applicable for solving (\ref{problem-cfl}). Adjusting GT-SAGA to
solve (\ref{problem-cfl}) can be realized by treating node $i$ in
Algorithm \ref{alg-2} as server $i$. In the CFL setting, despite
the fact that the training data are stored at users, we can
randomly select a user and let the user randomly pick a mini-batch
set of training samples to compute the local gradient. The local
gradient is then uploaded to the server to compute the VR-SG
$\boldsymbol{g}_i^{k+1}$. Mathematically, this can be written as
\begin{align}
\boldsymbol{g}_{i}^{k+1}=&\textstyle\tilde{S}_i\cdot
(\nabla f_{ij,t_{ij}^{k+1}}(\boldsymbol{x}_{i}^{k+1})-
\nabla f_{ij,t_{ij}^{k+1}}(\boldsymbol{\phi}_{ij,t_{ij}^{k+1}}^{k}))+
\nonumber\\
&\textstyle\sum_{j=1}^{P_i}
\sum_{t=1}^{S_{ij}}\nabla f_{ij,t}(\boldsymbol{\phi}_{ij,t}^{k}),
\end{align}
where $\tilde{S}_i\triangleq\sum_{j=1}^{P_i}S_{ij}$. It is easy
to verify that $\boldsymbol{g}_{i}^{k+1}$ is an
unbiased estimate of $\nabla f_{i}(\boldsymbol{x}_{i}^{k+1})$ if
both $j$ and $t_{ij}^{k+1}$ are uniformly selected, provided that
each user holds the same number of mini-batches. When the number
of mini-batches varies across different users, an unbiased
$\boldsymbol{g}_{i}^{k+1}$ can be obtained by assigning an
appropriate selection probability for each user. It is also
possible to select more than one user to participate in the
training. Under the assumption that each user holds the same
number of data samples, in Algorithm \ref{alg-3},
$\boldsymbol{g}_i^{k+1}$ is obtained by selecting
$|\bar{\mathcal{N}}_i^{k+1}|$ users, where $\bar{\mathcal{N}}_i^{k+1}$
is the index set of the selected users (by server $i$) in
the $(k+1)$th iteration.

The random user selection in Algorithm \ref{alg-3} provides a
convenient way to reduce the user-to-server uplink communication
overhead. Its random nature ensures the unbiasness of
$\boldsymbol{g}_i^{k+1}$. Thus the linear convergence rate of
Algorithm \ref{alg-3} can be obtained by using the theoretical
results in \cite{XinKhan20}. Despite the elegant linear
convergence rate of Algorithm \ref{alg-3}, it is unclear how to
determine the optimal number of users that are selected to upload
their gradients. Although a small number of selected users results
in a low per-iteration communication overhead, the required number
of iterations could be large since this leads to a large variance
in $\boldsymbol{g}_i^{k+1}$. Under such a fundamental tradeoff,
reducing the user sampling rate does not necessarily lead to
improved communication efficiency. Another drawback of random user
selection is that the selection is not based on the importance of
each local gradient. Thus the uploaded local gradients may not be
those most informative ones. This often leads to a degraded
convergence speed.

\begin{algorithm}
\caption{GT-SAGA for Confederated Learning}
\label{alg-3}
\begin{algorithmic}
\STATE{\textbf{Input}: $N$, $\{P_i\}$, $\boldsymbol{W}$, $\alpha$,
$\{\boldsymbol{x}_i^{0},\boldsymbol{y}_i^{0},\boldsymbol{\phi}_{ij,t}^{0}
,\boldsymbol{g}_{i,\text{sum}}^{0}=\boldsymbol{0}\}_{i=1,j=1,t=1}^{i=N,j=P_i,t=S_{ij}}$.}
\STATE{\textbf{while} not converge \textbf{do}}
\STATE{\quad\textbf{For} each edge server \textbf{parallel do}}
\noindent \STATE{ \quad\quad\textbf{1.} Server $i$ computes
$\boldsymbol{x}_i^{k+1}=\sum_{i=1}^N w_{ii'}
\boldsymbol{x}_{i'}^k-\alpha\boldsymbol{y}_i^k$ and then
broadcasts $\boldsymbol{x}_i^{k+1}$ to its associated users
as well as neighboring servers;}
\noindent \STATE{ \quad\quad\textbf{2.} Server $i$ randomly
selects a fixed number of users, whose index set is denoted
as $\bar{\mathcal{N}}_i^{k+1}$;}
\noindent \STATE{ \quad\quad\textbf{3.} If $j\in\bar{\mathcal{N}}_i^{k+1}$,
user $u_{ij}$ uniformly generates a random integer $t_{ij}^{k+1}$,
$1\leq t_{ij}^{k+1}\leq S_{ij}$, and then computes
\begin{align}
\textstyle\boldsymbol{\tilde{g}}_{ij}^{k+1}=\nabla f_{ij,t_{ij}^{k+1}}
(\boldsymbol{x}_{i}^{k+1})-\nabla f_{ij,t_{ij}^{k+1}}
(\boldsymbol{\phi}_{ij,t_{ij}^{k+1}}^{k}),
\end{align}
User $u_{ij}$ uploads $\boldsymbol{\tilde{g}}_{ij}^{k+1}$ to
the server. Set $\boldsymbol{\phi}_{ij,t_{ij}^{k+1}}^{k+1}=
\boldsymbol{x}_i^{k+1}$ and also set $\boldsymbol{\phi}_{ij,t}^{k+1}=
\boldsymbol{\phi}_{ij,t}^{k}$, $\forall t\neq t_{ij}^{k+1}$.}
\noindent \STATE{ \quad\quad\textbf{4.} Server $i$ computes
$\boldsymbol{y}_i^{k+1}=\sum_{i'=1}^N w_{ii'}
\boldsymbol{y}_{i'}^{k}+\boldsymbol{g}_i^{k+1}-\boldsymbol{g}_i^k$,
where}
\begin{align}
&\textstyle\boldsymbol{g}_i^{k+1}=\frac{\tilde{S}_i}
{|\bar{\mathcal{N}}_i^{k+1}|}\sum_{j\in\bar{\mathcal{N}}_i^{k+1}}
\boldsymbol{\tilde{g}}_{ij}^{k+1}+\boldsymbol{g}_{i,\text{sum}}^{k},
\nonumber\\
&\textstyle \boldsymbol{g}_{i,\text{sum}}^{k}\triangleq
\sum_{j=1}^{P_i}\sum_{t=1}^{S_{ij}}
\nabla f_{ij,t}(\boldsymbol{\phi}_{ij,t}^{k})
\end{align}
\noindent \STATE{ \quad\quad\textbf{5.} Server $i$ broadcasts
$\boldsymbol{y}_i^{k+1}$ to its neighboring servers and then updates}
\begin{align}
\textstyle\boldsymbol{g}_{i,\text{sum}}^{k+1}=\boldsymbol{g}_{i,\text{sum}}^{k}
+\sum_{j=1}^{P_i}\sum_{j\in\bar{\mathcal{N}}_i^{k+1}}
\boldsymbol{\tilde{g}}_{ij}^{k+1}
\end{align}
\par{\quad\textbf{End For};}
\par{\textbf{End while} and \textbf{Output $\boldsymbol{x}^{k+1}$};}
\end{algorithmic}
\end{algorithm}

\section{Proposed Algorithm}
\label{sec-4} Although the GT-SAGA can be adapted to solve the
CFL problem, it usually does not achieve optimal communication
efficiency due to the intrinsic limitations of random user
selection. In this section, we propose a communication-efficient
algorithm whose major innovation is the so called
conditionally-triggered user selection (CTUS). The proposed
algorithm meticulously selects a small number of users for
gradient uploading at each iteration and maintains a fast linear
convergence rate, thus leading to a higher communication
efficiency.

\subsection{Summary of Algorithm}
The proposed algorithm, abbreviated as CFL-SAGA
(\textbf{C}on\textbf{f}ederated \textbf{L}earning with \textbf{SAGA}),
is summarized in Algorithm \ref{alg-4}. In Algorithm \ref{alg-4},
Step $1$ and Step $5$ are similar to those in standard GT.
In Step $2$, the quantity $\|\sum_{i'=1}^{N}w_{ii'}\boldsymbol{x}_{i'}^{k+1}-
\boldsymbol{x}_i^{k+1}\|_2^2$ is computed and then sent to server
$i$'s users. This quantity is used by each user to determine whether
or not to upload its gradient. Step $3.1$ computes a local VR-SG
$\boldsymbol{g}_{ij}^{k+1}$ to provide an unbiased approximation of
$\nabla f_{ij}(\boldsymbol{x}_i^{k+1})$. The core innovation of
Algorithm \ref{alg-4} is Step $3.2$, namely, the CTUS step. This
step states that, for each user $u_{ij}$, the gradient innovation
vector
$\boldsymbol{\Delta}_{ij}^{k+1}=\boldsymbol{g}_{ij}^{k+1}-\boldsymbol{g}_{ij}^k$
should be uploaded to server $i$ only when the triggering
condition (\ref{main-alg-4}) is satisfied. At Step $4$, the
aggregated gradient $\boldsymbol{g}_i^{k+1}$ is obtained by summing
the newly uploaded user gradient $\boldsymbol{g}_{ij}^{k+1}$,
$j\notin \mathcal{N}_i^{k+1}$ as well as the stale user gradient
$\boldsymbol{g}_{ij}^{k}$, $j\in\mathcal{N}_i^{k+1}$. It should be
noted that, for server $i$, it does not need to store every individual
$\boldsymbol{g}_{ij}^k$. Instead, only the sum of all
$\boldsymbol{g}_{ij}^k$s needs to be stored.

\subsection{Rationale Behind The CTUS Mechanism}
\label{sec-rational}
Next, we discuss the rationale behind the CTUS mechanism. Without
loss of generality, we assume that $\rho=1$. For the right hand
side of the triggering condition (\ref{main-alg-4}), we deduce
that
\begin{align}
&\textstyle \sum_{i'=1}^{N}w_{ii'}\boldsymbol{x}_{i'}^{k+1}-
\boldsymbol{x}_i^{k+1}
\nonumber\\
=&\textstyle
\sum_{i'=1}^{N}w_{ii'}\boldsymbol{x}_{i'}^{k}-
\sum_{i'=1}^{N}w_{ii'}(\boldsymbol{x}_{i'}^{k}-
\boldsymbol{x}_{i'}^{k+1})-\boldsymbol{x}_i^{k+1}
\nonumber\\
\overset{(a)}{=}&\textstyle \alpha\boldsymbol{y}_i^k-
\sum_{i'=1}^{N}w_{ii'}(\boldsymbol{x}_{i'}^{k}-
\boldsymbol{x}_{i'}^{k+1})
\nonumber\\
=&\textstyle \underbrace{\textstyle
\alpha\boldsymbol{y}_i^{k}-\alpha\boldsymbol{y}_i^{k+1}-
\sum_{i'=1}^{N}w_{ii'}(\boldsymbol{x}_{i'}^{k}-
\boldsymbol{x}_{i'}^{k+1})}_{\text{[(\ref{discussion-2})-1]}}+
\alpha\boldsymbol{y}_i^{k+1} \label{discussion-2}
\end{align}
where $(a)$ is due to Step $1$ in Algorithm \ref{alg-4}. Note that
[(\ref{discussion-2})-1] is the difference between
$\boldsymbol{x}_{i}^{k+2}$ and $\boldsymbol{x}_{i}^{k+1}$, more
precisely,
\begin{align}
\boldsymbol{x}_i^{k+2}=&\textstyle\sum_{i'=1}^{N}w_{ii'}
\boldsymbol{x}_{i'}^{k+1}-\alpha\boldsymbol{y}_i^{k+1}
\nonumber\\
=&\textstyle\underbrace{\textstyle\sum_{i'=1}^{N}w_{ii'}\boldsymbol{x}_{i'}^{k}-
\alpha\boldsymbol{y}_i^{k}}_{=\boldsymbol{x}_i^{k+1}}+
\text{[(\ref{discussion-2})-1]}. \label{discussion-2-1}
\end{align}

Suppose the proposed algorithm converges to the true solution
$\boldsymbol{x}^*$ as $k\rightarrow\infty$. The DAC mechanism
ensures that $\boldsymbol{y}_i^{k+1}$ converges to
$\frac{1}{N}\sum_{i=1}^N\boldsymbol{g}_i^{k+1}$, which converges
to $\frac{1}{N}\sum_{i=1}^N\nabla f_i(\boldsymbol{x}^*) =0$ as
$k\rightarrow\infty$. As such, $\alpha\boldsymbol{y}_i^{k+1}$ in
(\ref{discussion-2}) can be rewritten as
$\alpha(\boldsymbol{y}_i^{k+1} -\boldsymbol{y}_i^{*})$, where
$\boldsymbol{y}_i^{*}= \boldsymbol{0}$ is the optimal
$\boldsymbol{y}_i$. Substituting this and (\ref{discussion-2-1})
into (\ref{discussion-2}) yields
\begin{align}
&\textstyle\sum_{i'=1}^{N}w_{ii'}\boldsymbol{x}_{i'}^{k+1}-
\boldsymbol{x}_i^{k+1}
=\underbrace{\boldsymbol{x}_i^{k+2}-
\boldsymbol{x}_i^{k+1}}_{\text{innovation of} \ \boldsymbol{x}_i}+
\alpha\underbrace{\big(\boldsymbol{y}_i^{k+1}-
\boldsymbol{y}_i^{*}\big)}_{\text{optimality gap of} \
\boldsymbol{y}_i}.
\nonumber
\end{align}
Clearly, both the innovation of $\boldsymbol{x}_i$ and the optimality
gap of $\boldsymbol{y}_i$ should converge to $\boldsymbol{0}$ as the
algorithm converges. From this perspective, the quantity
$\|\sum_{i'=1}^{N}w_{ii'}\boldsymbol{x}_{i'}^{k+1}-
\boldsymbol{x}_i^{k+1}\|_2$ approximately measures how much progress
can be made in the $(k+1)$th iteration. Therefore
\begin{align}
\textstyle
\|\boldsymbol{\Delta}_{ij}^{k+1}\|_2>
\|\sum_{i'=1}^{N}w_{ii'}\boldsymbol{x}_{i'}^{k+1}-\boldsymbol{x}_i^{k+1}\|_2,
\label{discussion-2-2}
\end{align}
indicates that $\boldsymbol{\Delta}_{ij}^{k+1}$ can make a
significant contribution to the updates of $\boldsymbol{x}_i^{k+2}$
and $\boldsymbol{y}_i^{k+1}$. For this case, Step $3.2$ suggests
$\boldsymbol{\Delta}_{ij}^{k+1}$ should be uploaded to the server.
Otherwise, $\boldsymbol{\Delta}_{ij}^{k+1}$ needs not to be uploaded
since it may not be sufficiently informative for the update of the
variables.

\begin{algorithm}
\caption{Proposed Algorithm: Confederated Learning with Stochastic
Average Gradient (CFL-SAGA)} \label{alg-4}
\begin{algorithmic}
\STATE{\textbf{Input}: $N$, $\{P_i\}$, $\boldsymbol{W}$, $\alpha$,
$\{\boldsymbol{x}_i^{0},\boldsymbol{y}_i^{0},\boldsymbol{\phi}_{ij,t}^{0}
=\boldsymbol{0}\}_{i=1,j=1,t=1}^{i=N,j=P_i,t=S_{ij}}$.}
\STATE{\textbf{while} not converge \textbf{do}}
\STATE{\quad\textbf{For} each edge server \textbf{parallel do}}
\noindent \STATE{ \quad\quad\textbf{1.} Server $i$ computes
$\boldsymbol{x}_i^{k+1}=\sum_{i'=1}^{N}w_{ii'}\boldsymbol{x}_{i'}^k-
\alpha\boldsymbol{y}_i^k$ and then broadcasts $\boldsymbol{x}_i^{k+1}$
to its associated users as well as its neighboring servers;}
\noindent \STATE{ \quad\quad\textbf{2.} Server $i$ computes
$\|\sum_{i'=1}^{N}w_{ii'}\boldsymbol{x}_{i'}^{k+1}-
\boldsymbol{x}_i^{k+1}\|_2^2$ and then broadcasts
this quantity to its associated users;}
\noindent \STATE{\qquad\textbf{For} each user \textbf{parallel do}}
\noindent \STATE{ \qquad\quad\textbf{3.1.} In user $u_{ij}$,
uniformly generate a random integer $t_{ij}^{k+1}$, $1\leq t_{ij}^{k+1}
\leq S_{ij}$, and then compute
\begin{align}
\textstyle\boldsymbol{g}_{ij}^{k+1}=&S_{ij}\cdot(\nabla f_{ij,t_{ij}^{k+1}}
(\boldsymbol{x}_{i}^{k+1})-\nabla f_{ij,t_{ij}^{k+1}}
(\boldsymbol{\phi}_{ij,t_{ij}^{k+1}}^{k}))+
\nonumber\\
&\textstyle
\sum_{t=1}^{S_{ij}}\nabla f_{ij,t}(\boldsymbol{\phi}_{ij,t}^{k}).
\label{main-alg-1}
\end{align}
Set $\boldsymbol{\phi}_{ij,t_{ij}^{k+1}}^{k+1}=\boldsymbol{x}_i^{k+1}$
and also set $\boldsymbol{\phi}_{ij,t}^{k+1}=\boldsymbol{\phi}_{ij,t}^{k}$,
$\forall t\neq t_{ij}^{k+1}$.}
\noindent \STATE{ \qquad\quad\textbf{3.2.} Let
$\boldsymbol{\Delta}_{ij}^{k+1}=\boldsymbol{g}_{ij}^{k+1}-
\boldsymbol{g}_{ij}^k$. Uploading $\boldsymbol{\Delta}_{ij}^{k+1}$
if
\begin{align}
\textstyle\|\boldsymbol{\Delta}_{ij}^{k+1}\|_2^2>
\rho\|\sum_{i'=1}^{N}w_{ii'}\boldsymbol{x}_{i'}^{k+1}-
\boldsymbol{x}_i^{k+1}\|_2^2.
\label{main-alg-4}
\end{align}
Let the users which satisfy (\ref{main-alg-4}) upload
$\boldsymbol{\Delta}_{ij}^{k+1}$ and denote $\mathcal{N}_{i}^{k+1}$
as the index set of users that does not satisfy (\ref{main-alg-4}); }
\noindent \STATE{\qquad\textbf{End For}}
\noindent \STATE{ \quad\quad\textbf{4.}
Server $i$ computes
\begin{align}
&\textstyle\boldsymbol{g}_i^{k+1}=
\sum_{j\in\mathcal{N}_{i}^{k+1}}\boldsymbol{g}_{ij}^k+
\sum_{j\notin\mathcal{N}_{i}^{k+1}}(\boldsymbol{g}_{ij}^k
+\boldsymbol{\Delta}_{ij}^{k+1}).
\label{main-alg-3}
\end{align}
}
\noindent \STATE{ \quad\quad\textbf{5.} Server $i$ computes
$\boldsymbol{y}_i^{k+1}=\sum_{i'=1}^{N}w_{ii'}\boldsymbol{y}_{i'}^k
+\boldsymbol{g}_i^{k+1}-\boldsymbol{g}_i^k$ and then broadcasts
$\boldsymbol{y}_i^{k+1}$ to its neighboring servers. }
\par{\quad\textbf{End For};}
\par{\textbf{End while} and \textbf{Output $\boldsymbol{x}^{k+1}$};}
\end{algorithmic}
\end{algorithm}

\subsection{Discussions}
\label{sec-IV-c}
The reuse of the stale user gradient is crucial to ensure the fast
convergence speed of the proposed algorithm. Thanks to the CTUS
mechanism, reusing the stale user gradient only leads to a
controllable error. Hence $\boldsymbol{g}_i^{k+1}$ in
(\ref{main-alg-3}) can be a close approximation of the aggregated
gradient. Moreover, in distributed optimization, the user gradient
usually changes slowly, especially in the high-precision regime.
Therefore it is reasonable to reuse the stale user gradient for
many iterations. Since only a small number of users are required
to upload their gradients, the proposed algorithm is expected to
achieve a high communication efficiency. Nevertheless, reusing
the stale user gradient in $\boldsymbol{g}_i^{k+1}$ breaks the
unbiasedness of $\boldsymbol{g}_i^{k+1}$, which brings difficulties
in proving the convergence of the algorithm.

The proposed CTUS mechanism is different from the
event-triggering-based schemes developed for standard FL
\cite{RichtarikSokolov22,ChenGiannakis18,ChenSun21,SunChen22}.
A distinctive feature of the triggering condition
for our proposed method is that it involves variables of neighboring
servers in order to quantify whether the local gradient is
informative enough for uploading. As discussed in Section
\ref{sec-rational}, the metric employed in (\ref{main-alg-4})
provides an estimate of the gap between the current solution and
the optimal point. Intuitively, a user should upload its local
gradient only if the local gradient is sufficiently informative
compared to the optimality gap. Since the optimality gap for the
CFL framework needs to account for the discrepancy between model
vectors of different servers, the event-triggering techniques
developed for standard federated learning systems
\cite{RichtarikSokolov22,ChenGiannakis18,ChenSun21,SunChen22} are
no longer applicable.

The proposed CTUS is also significantly different from
the triggering techniques designed for multi-agent decentralized
networks \cite{KiaCortes15,KajiyamaHayashi18,SinghData22,
GeorgeGurram20,GaoDeng21,ZehtabiHosseinalipour22}. In multi-agent
decentralized systems, the purpose of employing event-triggering
is to determine whether an agent should exchange its local
variables with its neighboring agents. In contrast, for our
proposed algorithm, communication between neighboring servers
is always assumed in every iteration, and the event-triggering
mechanism is mainly used to prune users that are deemed
unnecessary to upload their gradients to their respective servers.
Hence, both the purpose and the criterion of our proposed CTUS are
different from those of existing event-triggering methods.

\section{Convergence Results}
\label{sec-5} In this subsection, we aim to prove the linear
convergence rate of the proposed algorithm. Before proceeding to
the main result, we first introduce several notations. Define
$\boldsymbol{x}\triangleq[\boldsymbol{x}_1;
\cdots;\boldsymbol{x}_N]$ (resp.
$\boldsymbol{y}\triangleq[\boldsymbol{y}_1;
\cdots;\boldsymbol{y}_N]$) as the vertical stack of
$\boldsymbol{x}_i$s (resp. $\boldsymbol{y}_i$s). Let
$\bar{\boldsymbol{x}}\triangleq \frac{1}{N}
(\boldsymbol{1}_N^T\otimes \boldsymbol{I}_d)\boldsymbol{x}$ (resp.
$\bar{\boldsymbol{y}}\triangleq
\frac{1}{N}(\boldsymbol{1}_N^T\otimes
\boldsymbol{I}_d)\boldsymbol{y}$) be the average of
$\boldsymbol{x}_i$s (resp. $\boldsymbol{y}_i$s). Also define
$\bar{\boldsymbol{W}}
\triangleq\boldsymbol{W}\otimes\boldsymbol{I}_{d}$ and
$\bar{\boldsymbol{W}}_{\infty}\triangleq\boldsymbol{W}_{\infty}\otimes
\boldsymbol{I}_{d}$, where
$\boldsymbol{W}_{\infty}=\frac{\boldsymbol{1}_{N}
\boldsymbol{1}_{N}^T}{I}$ and $\otimes$ represents the Kronecker
product. The convergence result for Algorithm \ref{alg-4} is
summarized in the following theorem.

\newtheorem{theorem}{Theorem}
\begin{theorem}
\label{theorem-1}
Let $\boldsymbol{x}^*$ denote the optimal solution to the CFL
problem (\ref{problem-cfl}). Assume that the objective function
(resp. server network) satisfies the assumptions made in Section
\ref{sec-obj-assumption}. Define
\begin{align}
\boldsymbol{\psi}^{k}=\big[\mathbb{E}\{X^k\};\mathbb{E}\{\bar{X}^k\};
\mathbb{E}\{D^{k-1}\};\mathbb{E}\{Y^k\}\big]
\label{theorem-content-1}
\end{align}
where $X^k\triangleq \|\boldsymbol{x}^{k}-
\bar{\boldsymbol{W}}_{\infty}\boldsymbol{x}^{k}\|_2^2$, $\bar{X}^k
\triangleq\|\bar{\boldsymbol{x}}^{k}-\boldsymbol{x}^*\|_2^2$,
$Y^k\triangleq\|\boldsymbol{y}^{k}-\bar{\boldsymbol{W}}_{\infty}
\boldsymbol{y}^{k}\|_2^2$,
\begin{align}
&\textstyle D^k\triangleq\sum_{i=1}^{N}\sum_{j=1}^{P_i}\sum_{t_{ij}^k=1}^{S_{ij}}
\|\boldsymbol{x}^*-\boldsymbol{\phi}_{ij,t_{ij}^k}^{k}\|_2^2,
\end{align}
with $\boldsymbol{x}^{k+1}$, $\boldsymbol{y}^{k+1}$, and
$\{\boldsymbol{\phi}_{ij,t}^{k}\}_{i,j,t}$ generated by Algorithm
\ref{alg-4}. If the stepsize $\alpha$ is chosen to be
sufficiently small, we have
\begin{align}
\textstyle
\boldsymbol{\psi}^{k+1}\leq \underbrace{\textstyle(1-\frac{\mu\alpha}{4}+
\frac{2c_2\alpha^2}{N})}_{\triangleq\gamma}\cdot\boldsymbol{\psi}^{k}
\label{theorem-content-1-1}
\end{align}
where $c_2\triangleq 8L^2(1+\overline{P_i}\cdot\overline{S_{ij}}^2)N$,
$\overline{S_{ij}}=\max_{i,j}\{S_{ij}\}$ and
$\overline{P_i}=\max_{i}\{P_i\}$ are constants, $N$
is the number of servers, $P_i$ is the number of users
associated with server $i$ and $L$ is the Lipschitz constant.
\end{theorem}

In Theorem \ref{theorem-1}, the metric to characterize the
convergence behavior of the proposed algorithm is
$\boldsymbol{\psi}^{k}$. In $\boldsymbol{\psi}^{k}$, $X^k$ (resp.
$Y^k$) is the consensus gap measuring the distance between the
server-side local variable $\boldsymbol{x}^i$ (resp.
$\boldsymbol{y}^i$) and the average of the local variables, i.e.,
$\bar{\boldsymbol{x}}^k$ (resp. $\bar{\boldsymbol{y}}^k$).
When $X^k=0$ (resp. $Y^k=0$), it means that consensus
among servers is achieved. The metric $\bar{X}^k$ measures the
distance between $\bar{\boldsymbol{x}}^k$ and the optimal point
$\boldsymbol{x}^*$. Clearly, each local variable
$\boldsymbol{x}^i$ converges to the optimal point
$\boldsymbol{x}^*$ when both consensuses are achieved and
$\bar{X}^k=0$. The metric $D^k$ measures the distance between
$\boldsymbol{\phi}_{ij,t_{ij}^k}^{k}\in\mathbb{R}^{d}$, which is
the local variable corresponding to the $t_{ij}^k$th training
sample at user $u_{ij}$, and the optimal point
$\boldsymbol{x}^*$. Therefore $D^k=0$ indicates that all the
user-side local variables also converge to the optimal point. To
conclude, the metric $\boldsymbol{\psi}^{k}$ characterizes the
convergence behavior of the proposed algorithm from different
perspectives, say, consensus achieving as well as optimality
reaching. As such, $\boldsymbol{\psi}^{k}=0$ implies that the
proposed algorithm has already arrived at the optimal point.

The inequality (\ref{theorem-content-1-1}) indicates a linear
convergence rate of Algorithm \ref{alg-4}, provided that the rate
$\gamma<1$. This is always achievable if we set $\alpha$ to be
sufficiently close to $0$ because the second-order polynomial
$\frac{2c_2\alpha^2}{N}$ decreases faster than the first-order
polynomial $\frac{\mu\alpha}{4}$. Our theoretical result can be
considered as a generalization of the result in \cite{XinKhan20}.
Such an extension, however, is highly nontrivial as the CTUS
mechanism breaks the unbiasness of the server-side local gradient.

\section{Proof of Theorem \ref{theorem-1}}
\label{sec-6}
In appendices, we proved four different
inequalities. Combining those inequalities yields the following
vector-form inequality:
\begin{align}
\boldsymbol{\psi}^{k+1}\leq \boldsymbol{T} \boldsymbol{\psi}^{k},
\label{proofn-2}
\end{align}
where $\boldsymbol{\psi}^{k+1}$ is defined in
(\ref{theorem-content-1}), and
\begin{align}
\boldsymbol{T}=
\begin{bmatrix}
\frac{1+\sigma^2}{2} & 0 & 0 & \frac{2\alpha^2}{1-\sigma^2}\\
b_2 & b_1 & b_3 & 0\\
2\overline{P_i} & 2\overline{P_i}N &
(1-\frac{1}{\overline{S_{ij}}}) & 0\\
a_1 & a_2 & a_4  & a_3
\end{bmatrix}
\label{proofn-2-1}
\end{align}
Before introducing the notations in (\ref{proofn-2-1}), first
notice that the $4$ inequalities in (\ref{proofn-2}),
from top to bottom order, are respectively proved in Lemma
\ref{lemma10} (see (\ref{lemma10-content})), Lemma \ref{lemma5}
(see (\ref{lemma5-1-content1}), Appendix \ref{appendix-B}), Lemma
\ref{lemma5-1} (see (\ref{lemma5-1-content}), Appendix
\ref{appendix-C}) and Lemma \ref{lemma9} (see
(\ref{lemma9-content1}), Appendix \ref{appendix-D}).

We now define the notations in (\ref{proofn-2-1}). Let
$\underline{S_{ij}}\triangleq \min_{i,j}\{S_{ij}\}$,
$\overline{S_{ij}} \triangleq\max_{i,j}\{S_{ij}\}$ and
$\overline{P_i}\triangleq\max_{i}\{P_{i}\}$. $\sigma<1$ is
the second largest singular value of the mixing matrix
$\boldsymbol{W}$ and $\alpha$ is the stepsize parameter in
Algorithm \ref{alg-4}. Also, $\{b_i\}_{i=1}^3$ and
$\{a_i\}_{i=1}^4$ in $\boldsymbol{T}$ are defined as
\begin{align}
b_1=&\textstyle 1-\mu\alpha+\frac{2c_2\alpha^2}{N},
b_2=\frac{2\alpha L^2+4\alpha \rho(1+\sigma^2)\overline{P_i}^2+
2\mu\alpha^2L^2+2\mu\alpha^2c_1}{\mu N}, \nonumber\\
b_3=&\textstyle\frac{2\alpha^2c_3}{N}, \\
&\textstyle a_1=\frac{25L^2(1+\sigma^2)+4(1+\sigma^2)c_1
+3(1+\sigma^2)(2c_1\sigma^2+c_2\bar{b}+
2c_3\overline{P_i}N)}{1-\sigma^2},
\nonumber\\
&\textstyle
a_2=\frac{NL^2(1+\sigma^2)+4(1+\sigma^2)c_2+
3(1+\sigma^2)(c_2(2+\frac{2c_2}{L^2})+
2c_3\overline{P_i}N)}{1-\sigma^2},
\nonumber\\
&\textstyle a_3=\frac{1+\sigma^2}{2}+
\frac{24\alpha^2 L^2(1+\sigma^2)}{1-\sigma^2}
+\frac{6\alpha^2c_1(1+\sigma^2)}{1-\sigma^2},
\nonumber\\
&\textstyle
a_4=\frac{4(1+\sigma^2)c_3}{1-\sigma^2}
+\frac{3(1+\sigma^2)(b_3c_2+c_3(1-\overline{S_{ij}}^{-1}))}{1-\sigma^2}
\end{align}
in which $c_i$, $1\leq i\leq 3$, is a constant number
(see (\ref{lemma4-1-content2})), and $\bar{b}$ in $a_1$ is defined as
\begin{align}
\textstyle \bar{b}=\frac{4\alpha^2L^2+4\alpha^2\rho(1+\sigma^2)
\overline{P_i}^2+2\alpha^2c_1}{N}.
\end{align}

From (\ref{proofn-2}), proving linear convergence rate of
Algorithm \ref{alg-4} is equivalent to proving
$\rho(\boldsymbol{T})<1$. According to Lemma \ref{lemma6}
provided in Appendix A, if we can find a positive vector
$\boldsymbol{\psi}\in\mathbb{R}^{4}$ such that
$\boldsymbol{T}\boldsymbol{\psi}\leq \gamma \boldsymbol{\psi}$
holds with $\gamma<1$, then we have $\rho(\boldsymbol{T})<1$. To
do this, set $\gamma=1-\frac{\mu\alpha}{4}+
\frac{2c_2\alpha^2}{N}$ and from now on we are going to find a
positive vector $\boldsymbol{\psi}$ such that $\boldsymbol{T}
\boldsymbol{\psi}\leq\gamma \boldsymbol{\psi}$. Recall that
the parameters are chosen such that $\gamma$ is guaranteed to
be smaller than $1$.

\subsection{Finding $\boldsymbol{\psi}$}
Using elementary algebra, it is easy to deduce that
$\boldsymbol{T}\boldsymbol{\psi} \leq \gamma \boldsymbol{\psi}$ is
equivalent to the following set of inequalities:
\begin{align}
&\textstyle \frac{\mu\alpha}{4}-\alpha^2\underbrace{\textstyle\big(\frac{2c_2}{N}-
\frac{2}{1-\sigma^2}
\frac{[\boldsymbol{\psi}]_4}{[\boldsymbol{\psi}]_1}\big)}_{\text{[(\ref{proof-17-1})-1]}}
\leq \frac{1-\sigma^2}{2},
\label{proof-17-1}\\
&\textstyle \frac{2\alpha^2c_3}{N}\cdot [\boldsymbol{\psi}]_3
\leq \underbrace{\textstyle\frac{3}{4}\mu\cdot [\boldsymbol{\psi}]_2-
\frac{b_2}{\alpha}\cdot[\boldsymbol{\psi}]_1}_{\text{[(\ref{proof-17-2})-1]}},
\label{proof-17-2}\\
&\textstyle 2\overline{P_i} [\boldsymbol{\psi}]_1+
2\overline{P_i}N [\boldsymbol{\psi}]_2
\leq \underbrace{\textstyle\big(1-\frac{\mu\alpha}{4}+
\frac{2c_2\alpha^2}{N}-(1-\frac{1}{\overline{S_{ij}}})\big)}_{\text{[(\ref{proof-17-3})-1]}}
[\boldsymbol{\psi}]_3,
\label{proof-17-3}\\
&\textstyle a_1\cdot[\boldsymbol{\psi}]_1+a_2\cdot[\boldsymbol{\psi}]_2+
a_4\cdot[\boldsymbol{\psi}]_3\leq
\underbrace{\textstyle(1-\frac{\mu\alpha}{4}
+\frac{2c_2\alpha^2}{N}-a_3)}_{\text{[(\ref{proof-17-4})-1]}}
[\boldsymbol{\psi}]_4,
\label{proof-17-4}
\end{align}
where we used the definition of $\gamma\triangleq1-\frac{\mu\alpha}{4}+
\frac{2c_2\alpha^2}{N}$ and that of $\boldsymbol{T}$. In
the above, $[\boldsymbol{\psi}]_i$ represents the $i$th element of
$\boldsymbol{\psi}$. Since $\boldsymbol{T}$ is positive and
$\boldsymbol{\psi}$ should be positive, we need to first ensure
the positiveness of [(\ref{proof-17-3})-1] and
[(\ref{proof-17-4})-1].

\subsubsection{Ensuring positiveness of [(\ref{proof-17-3})-1]}
It is easy to see that $\text{[(\ref{proof-17-3})-1]}>0$ is equivalent to
$\frac{\mu\alpha}{4}-\frac{2c_2\alpha^2}{N}<\frac{1}{\overline{S_{ij}}}$.
Set $\alpha_{\text{thresh},3}$ such that
$\frac{\mu\alpha_{\text{thresh},3}}{4}=\frac{1}{\overline{S_{ij}}}$.
It is easy to verify that
$\frac{\mu\alpha}{4}-\frac{2c_2\alpha^2}{N}
<\frac{1}{\overline{S_{ij}}}$ can always be guaranteed if
$\alpha\in(0,\alpha_{\text{thresh},3})$.

\subsubsection{Ensuring positiveness of $\text{[(\ref{proof-17-4})-1]}$}
According to the definition of $a_3$, we have
\begin{align}
&\text{[(\ref{proof-17-4})-1]}=\textstyle \frac{1-\sigma^2}{2}-
\frac{\mu\alpha}{4}+\frac{2c_2\alpha^2}{N}-
\frac{24\alpha^2 L^2(1+\sigma^2)}{1-\sigma^2}
-\frac{6\alpha^2c_1(1+\sigma^2)}{1-\sigma^2}
\nonumber\\
&=\textstyle \frac{1-\sigma^2}{2}-\frac{\mu\alpha}{4}+
\alpha^2\underbrace{\textstyle\big(\frac{2c_2}{N}-
\frac{24 L^2(1+\sigma^2)}{1-\sigma^2}-
\frac{6c_1(1+\sigma^2)}{1-\sigma^2}\big)}_{<0},
\label{proof-18}
\end{align}
where $<0$ can be verified by checking the definitions of $c_1$
and $c_2$. Since $0<\sigma<1$, we know that
$\frac{1-\sigma^2}{2}>0$. Set $\alpha_{\text{thresh},4}$ such that
$[(\ref{proof-17-4})-1]=0$ when $\alpha=\alpha_{\text{thresh},4}$.
Hence $\text{[(\ref{proof-17-4})-1]}>0$ can be ensured if
$\alpha\in(0,\alpha_{\text{thresh},4})$.

Suppose
$0<\alpha<\min\{\alpha_{\text{thres},3},\alpha_{\text{thres},4}\}$.
Then we have $\text{[(\ref{proof-17-3})-1]}>0$ and
$\text{[(\ref{proof-17-4})-1]}>0$. Next, we discuss how to determine
$[\boldsymbol{\psi}]_i$, $1\leq i\leq 4$.

\subsubsection{Determining $[\boldsymbol{\psi}]_1$ and $[\boldsymbol{\psi}]_2$}
To begin with, let $[\boldsymbol{\psi}]_1$ be an arbitrary
positive value. With $[\boldsymbol{\psi}]_1$ fixed, we can
find a sufficiently large $[\boldsymbol{\psi}]_2$ such that
[(\ref{proof-17-2})-1] is positive. This is because
\begin{align}
\textstyle\frac{b_2}{\alpha}=\frac{2L^2+4 \rho(1+\sigma^2)\overline{P_i}^2+
2\mu\alpha L^2+2\mu\alpha c_1}{\mu N}
\label{proof-18-1}
\end{align}
is upper bounded (since $\alpha$ is upper bounded).

\subsubsection{Determining $[\boldsymbol{\psi}]_3$}
With $[\boldsymbol{\psi}]_2$ and $[\boldsymbol{\psi}]_1$ fixed, we
can always choose a sufficiently large $[\boldsymbol{\psi}]_3$
such that (\ref{proof-17-3}) holds. Now since $[\boldsymbol{\psi}]_1$,
$[\boldsymbol{\psi}]_1$ and $[\boldsymbol{\psi}]_3$ are fixed,
(\ref{proof-17-2}) is guaranteed if
$\alpha\in(0,\alpha_{\text{thresh},2})$, where
$\alpha_{\text{thresh},2}$ satisfies
\begin{align}
&\textstyle \frac{2\alpha_{\text{thresh},2}^2 c_3}{N}\cdot
[\boldsymbol{\psi}]_3+\frac{b_2}{\alpha_{\text{thresh},2}}
\cdot[\boldsymbol{\psi}]_1
=\frac{3}{4}\mu\cdot [\boldsymbol{\psi}]_2.
\end{align}
Note that $\frac{b_2}{\alpha}$ is a polynomial of $\alpha$
(see (\ref{proof-18-1})), which means that $\frac{b_2}{\alpha}$
decreases as $\alpha$ decreases.

\subsubsection{Determining $[\boldsymbol{\psi}]_4$}
With $[\boldsymbol{\psi}]_1$, $[\boldsymbol{\psi}]_2$ and
$[\boldsymbol{\psi}]_3$ fixed as well as
$0<\alpha<\min_{i=2,3,4}\{\alpha_{\text{thres},i}\}$ (which means
that $[(\ref{proof-17-4})-1]>0$), there always exists a
sufficiently large $[\boldsymbol{\psi}]_4$ such that
(\ref{proof-17-4}) holds. Given a fixed $[\boldsymbol{\psi}]_1$
and $[\boldsymbol{\psi}]_4$, (\ref{proof-17-1}) can be guaranteed
by choosing a sufficiently small $\alpha$, no matter
$[(\ref{proof-17-1})-1]$ is positive or negative. The feasible
range of $\alpha$ for achieving this is denoted as
$(0,\alpha_{\text{thres},4})$.

Based on the above discussion, there exists $\boldsymbol{\psi}>0$
such that $\boldsymbol{T}\boldsymbol{\psi}\leq
\gamma\boldsymbol{\psi}$, provided that $\alpha<\min_{i=1,2,3,4}
\{\alpha_{\text{thres},i}\}$. As such, the spectral radius of
$\boldsymbol{T}$, $\rho(\boldsymbol{T})$, is no larger than
$\gamma=1-\frac{\mu\alpha}{4}+\frac{2c_2\alpha^2}{N}$. Combining
this with (\ref{proofn-2}) yields
\begin{align}
\textstyle \boldsymbol{\psi}^{k+1}\leq \boldsymbol{T}
\boldsymbol{\psi}^{k} \leq  \rho(\boldsymbol{T})
\cdot\boldsymbol{\psi}^{k} \leq (1-\frac{\mu\alpha}{4}+
\frac{2c_2\alpha^2}{N})\cdot\boldsymbol{\psi}^{k} \label{proof-19}
\end{align}
where the second inequality is obtained by realizing that
$\rho(\boldsymbol{T})$ is the largest absolute eigenvalue and
$\boldsymbol{T} \boldsymbol{\psi}^{k}$ is non-negative. Clearly,
(\ref{proof-19}) is the desired result.

\section{A Further Analysis of CTUS}
\label{sec-7} In this section, we provide a rigorous analysis to
show that the proposed CTUS mechanism can prune user uploads. This
is equivalent to showing that for a proper choice of $\rho$, the
triggering condition (\ref{main-alg-4}) dose not hold for a number
of users. Notice that the quantities on both sides of
(\ref{main-alg-4}) are random variables. Therefore if the
following inequality holds
\begin{align}
\textstyle \mathbb{E}\{\|\boldsymbol{\Delta}_{ij}^{k+1}\|_2^2\}<
\mathbb{E}\{\rho\|\sum_{i'=1}^{N}w_{ii'}\boldsymbol{x}_{i'}^{k+1}-
\boldsymbol{x}_i^{k+1}\|_2^2\},
\label{commu-redu-1}
\end{align}
then we can safely claim that
\begin{align}
\textstyle\mathbb{P}\Big\{\|\boldsymbol{\Delta}_{ij}^{k+1}\|_2^2<
\rho\|\sum_{i'=1}^{N}w_{ii'}\boldsymbol{x}_{i'}^{k+1}-
\boldsymbol{x}_i^{k+1}\|_2^2\Big\}\neq 0,
\label{commu-redu-2}
\end{align}
where $\mathbb{P}\{\cdot\}$ denotes the probability of an event,
and the expectation in (\ref{commu-redu-1}) is taken w.r.t. all
the random variables appeared up to the $k+1$th iteration. If
(\ref{commu-redu-2}) holds true, it means that user $u_{ij}$
has a nonzero probability not to upload its local gradient. To
show (\ref{commu-redu-1}) (for some $\rho$), we consider an
averaged version of (\ref{commu-redu-1}), that is,
\begin{align}
&\textstyle\frac{1}{\sum_{i=1}^N P_i}\cdot
\mathbb{E}\{\sum_{i=1}^N\sum_{j=1}^{P_i}
\|\boldsymbol{\Delta}_{ij}^{k+1}\|_2^2\}
\nonumber\\
\leq &\textstyle \frac{1}{\sum_{i=1}^N P_i}\cdot
\mathbb{E}\Big\{\sum_{i=1}^N
\rho P_i\|\sum_{i'=1}^{N}w_{ii'}\boldsymbol{x}_{i'}^{k+1}-
\boldsymbol{x}_i^{k+1}\|_2^2\Big\},
\label{commu-redu-3}
\end{align}
where the average is taken for all users. Clearly, if
(\ref{commu-redu-3}) holds true, then there must exist users
which satisfy (\ref{commu-redu-1}). To facilitate the
analysis, we suppose the sequence generated by Algorithm
\ref{alg-4} has reached to a point that is close to the optimal
solution, in which case we have the following result.

\newtheorem{proposition}{Proposition}
\begin{proposition}
\label{Proposition-1} Let $\boldsymbol{x}^*$ denote the optimal
solution to the CFL problem (\ref{problem-cfl}). Suppose we have
$D^{k+1}\approx D^k$ and
\begin{align}
&\|\bar{\boldsymbol{W}}_{\infty}\boldsymbol{x}^{k+1}-
\tilde{\boldsymbol{x}}^*\|_2^2\leq C_1\|\boldsymbol{x}^{k+1}-
\bar{\boldsymbol{W}}_{\infty}\boldsymbol{x}^{k+1}\|_2^2,
\label{theorem2-content-1}\\
&\|\bar{\boldsymbol{W}}\boldsymbol{x}^{k+1}-
\bar{\boldsymbol{W}}_{\infty}\boldsymbol{x}^{k+1}\|_2^2
\leq C_2\|\boldsymbol{x}^{k+1}-
\bar{\boldsymbol{W}}\boldsymbol{x}^{k+1}\|_2^2,
\label{theorem2-content-2}
\end{align}
where $\tilde{\boldsymbol{x}}^*\triangleq[\boldsymbol{x}^*;
\cdots;\boldsymbol{x}^*]$ is a vertical stack of $N$
$\boldsymbol{x}^*$s, and $C_1$ (resp. $C_2$) is a positive
constant. Then we have
\begin{align}
&\textstyle\frac{1}{\sum_{i=1}^N P_i}\cdot
\mathbb{E}\{\sum_{i=1}^N\sum_{j=1}^{P_i}
\|\boldsymbol{\Delta}_{ij}^{k+1}\|_2^2\}
\nonumber\\
\lesssim &\textstyle \mathbb{E}\Big\{
\frac{C_3}{N}\sum_{i=1}^N
\|\sum_{i'=1}^{N}w_{ii'}\boldsymbol{x}_{i'}^{k+1}-
\boldsymbol{x}_i^{k+1}\|_2^2\Big\},
\label{theorem2-content-3}
\end{align}
where $C_3\triangleq\frac{16(1+C_1)(1+C_2)\overline{S_{ij}}
\bar{L}\overline{P_i}N}{\sum_{i=1}^N P_i}$, and $\bar{L}$
is a constant defined in (\ref{comm-effici-3-3}).
\end{proposition}
\begin{proof}
See Appendix \ref{appendix-proposi}.
\end{proof}

From (\ref{theorem2-content-3}) we know that (\ref{commu-redu-3})
holds when $\rho$ is set to a sufficiently large value. Note that
the assumptions made in the above proposition are reasonable. To
see this, first recall that (\ref{theorem-content-1-1}) indicates
$\mathbb{E}\{D^{k+1}\}\leq \gamma\cdot\mathbb{E}\{D^{k}\}$. Since
$\gamma$ is usually close to $1$, we can safely assume that
$D^{k+1}\approx D^k$. Condition (\ref{theorem2-content-1}) holds
valid when $\boldsymbol{x}^{k+1}$ is sufficiently close to the
optimal point. Condition (\ref{theorem2-content-2}) can be
justified as follows. Recall that the left hand side of
(\ref{theorem2-content-2}) can be bounded by
\begin{align}
&\|\bar{\boldsymbol{W}}\boldsymbol{x}^{k+1}-
\bar{\boldsymbol{W}}_{\infty}\boldsymbol{x}^{k+1}\|_2^2
\overset{(a)}{=}\|\bar{\boldsymbol{W}}(\boldsymbol{x}^{k+1}-
\bar{\boldsymbol{W}}_{\infty}\boldsymbol{x}^{k+1})\|_2^2
\nonumber\\
&\overset{(b)}{\leq} \|\boldsymbol{x}^{k+1}-
\bar{\boldsymbol{W}}_{\infty}\boldsymbol{x}^{k+1}\|_2^2
\label{theorem2-discuss-1}
\end{align}
where $(a)$ is because $\bar{\boldsymbol{W}}\bar{\boldsymbol{W}}_{\infty}
=\bar{\boldsymbol{W}}_{\infty}$, and $(b)$ is because
$\|\bar{\boldsymbol{W}}\|_2=1$. Recall that the
right hand side of (\ref{theorem2-content-2}) is the discrepancy
between each local variable and the average of its neighboring
variables. While the right hand side of (\ref{theorem2-discuss-1})
is the discrepancy between each local variable and the average of
all local variables. These two quantities should be close provided
that consensus is nearly achieved. Combining
(\ref{theorem2-discuss-1}) and
\begin{align}
\|\boldsymbol{x}^{k+1}-\bar{\boldsymbol{W}}_{\infty}\boldsymbol{x}^{k+1}\|_2^2
\approx \|\boldsymbol{x}^{k+1}-
\bar{\boldsymbol{W}}\boldsymbol{x}^{k+1}\|_2^2,
\end{align}
we can safely assume that (\ref{theorem2-content-2}) holds.

\begin{figure}
  \centering
  \includegraphics[width=6cm,height=4cm]{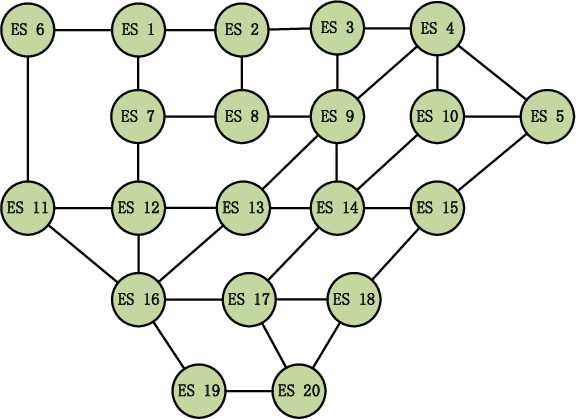}
  \caption{Topology of the server network}
  \label{fig1}
\end{figure}

\begin{figure*}[!htbp]
    \centering
    \subfigure[\scriptsize{Random graph.}]
    {\includegraphics[width=5cm,height=4cm]{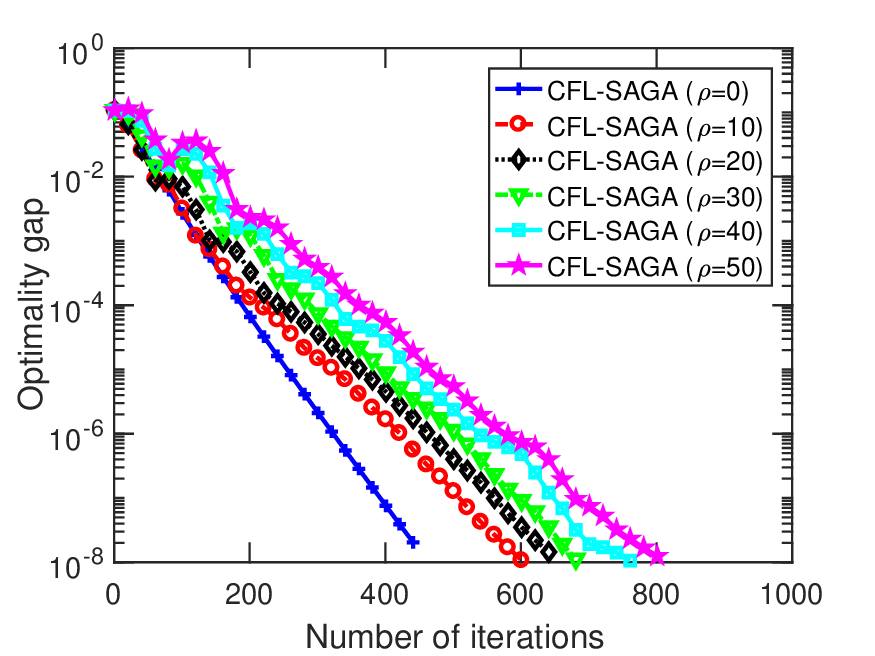}} \hfil
    \subfigure[\scriptsize{Ring graph.}]
    {\includegraphics[width=5cm,height=4cm]{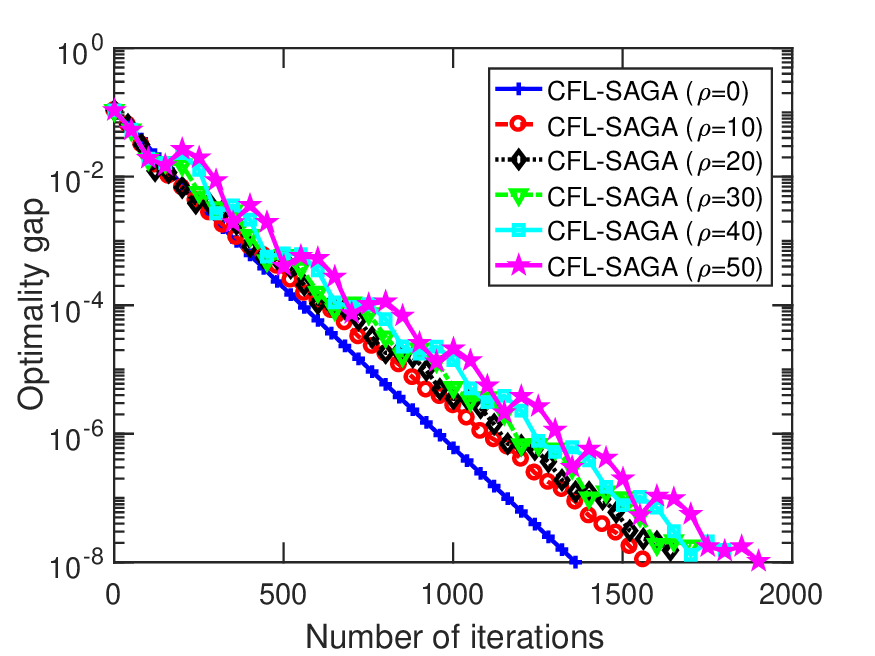}} \hfil
    \subfigure[\scriptsize{Fully connected graph.}]
    {\includegraphics[width=5cm,height=4cm]{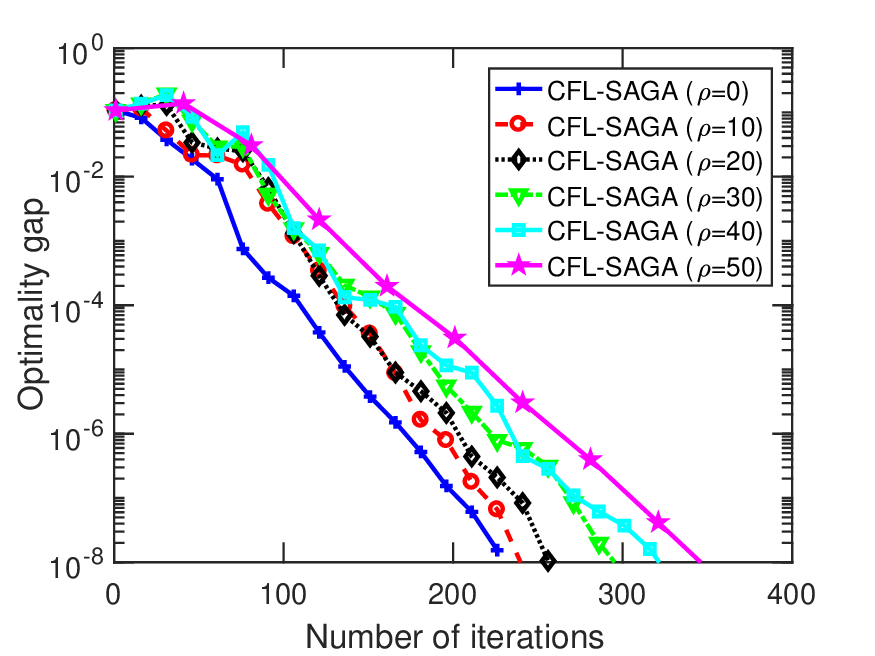}}
    \\
    \subfigure[\scriptsize{Random graph.}]
    {\includegraphics[width=5cm,height=4cm]{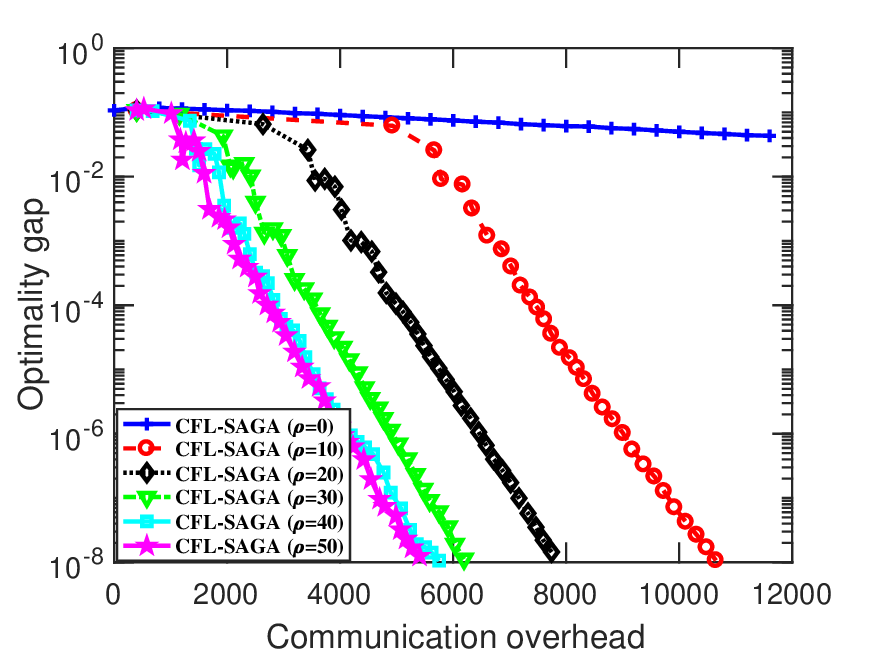}} \hfil
    \subfigure[\scriptsize{Ring graph.}]
    {\includegraphics[width=5cm,height=4cm]{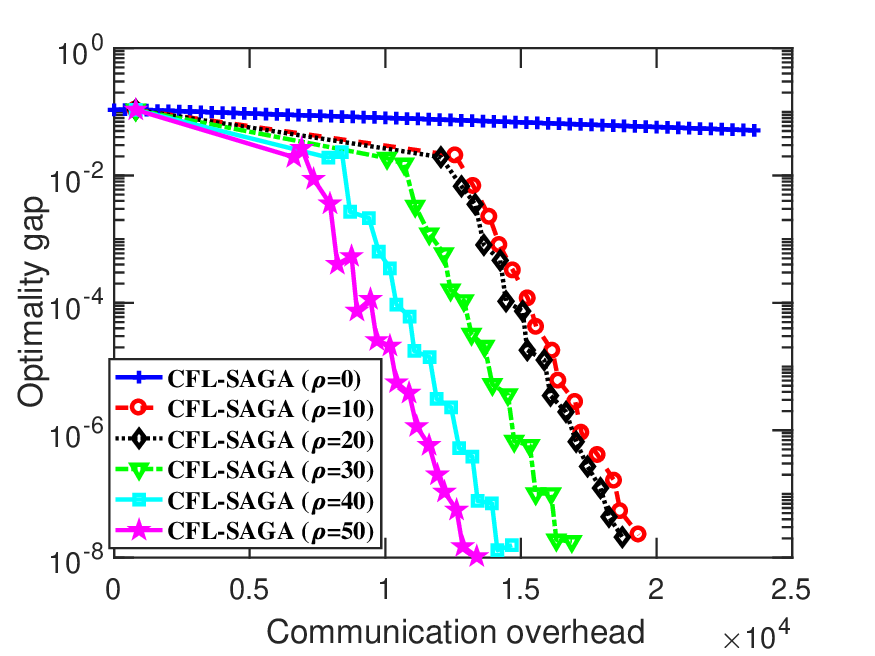}} \hfil
    \subfigure[\scriptsize{Fully connected graph.}]
    {\includegraphics[width=5cm,height=4cm]{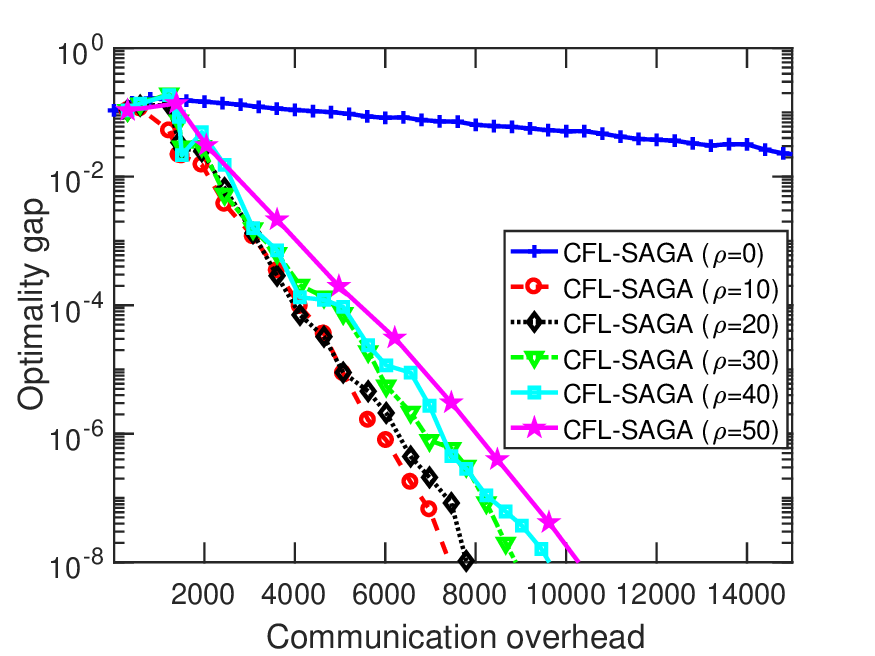}}
    \caption{Results on $\ell_2$-regularized logistic regression.
    First row: Optimality gap vs. number of iterations on different
    server networks; Second row: Optimality gap vs. communication overhead
    on different server networks}
    \label{fig2}
\end{figure*}

\section{Simulation Results}
\label{sec-8} In this section, we provide simulation results to
demonstrate the superiority of the proposed CFL-SAGA algorithm
over GT-SAGA and other competing algorithms. We compare the
performance of respective algorithms over different user sampling
rates as well as different server topologies. We first introduce
the experimental settings.

\subsection{Experimental Settings}
We consider a CFL system consisting of $N=20$ servers and $400$
users. Each server is assigned to $P_i=20$ users. To investigate
the performance of respective algorithms over different server
topologies, we consider three types of server networks. The first
one is a random graph depicted in Fig. \ref{fig1}. The second one
is the ring graph (cycle graph), and the last one is a fully
connected graph. Since the structures of the ring graph and the
fully connected graph are self-evident by their names, we do not
draw their topologies for sake of simplicity. Clearly, the fully
connected graph has the best connectivity, while the ring graph
has the poorest connectivity.

We adopt the $\ell_2$-regularized logistic regression problem as
our test problem:
\begin{align}
\mathop {\min }\limits_{\boldsymbol{x}\in\mathbb{R}^d}
& \ \textstyle f(\boldsymbol{x})\triangleq
\frac{1}{N}\sum_{i=1}^{N}\sum_{j=1}^{P_{i}}
f_{ij}(\boldsymbol{x}),
\label{simu-1}
\end{align}
where $f_{ij}(\boldsymbol{x})=\sum_{t=1}^{S_{ij}}
f_{ij,t}(\boldsymbol{x})$,
\begin{align}
 f_{ij,t}(\boldsymbol{x})=&\textstyle\sum\limits_{t'\in \mathcal{T}_{ij,t}}
\Big( \frac{\kappa}{2}\|\boldsymbol{x}\|_2^2 -y_{ij,t'}
\cdot\text{log}\big((1+e^{-\boldsymbol{\omega}_{ij,t'}^T\boldsymbol{x}})^{-1}\big)-
\nonumber\\
&(1-y_{ij,t'})\cdot\text{log}\big(1-(1+e^{-\boldsymbol{\omega}_{ij,t'}^T
\boldsymbol{x}})^{-1}\big)\Big)
\label{simu-2}
\end{align}
in which $\kappa$ is set to $0.05$, $\{\boldsymbol{\omega}_{ij,t'}
\in\mathbb{R}^{d},y_{ij,t'}\in\{0,1\}\}$ is the $t'$th training
sample stored at user $u_{ij}$, and $\mathcal{T}_{ij,t}$ is the
index set of the data samples in the $t$th mini-batch training
set. Clearly, $f_{ij,t}$ is strongly convex and its gradient is
Lipschitz continuous. In our experiments, both the data vector
$\boldsymbol{\omega}_{ij,t'}\in \mathbb{R}^{200}$ and the label
$y_{ij,t'}\in\{0,1\}$ are randomly generated. Each user is assumed
to hold $50$ training samples and each mini-batch training set
consists of $5$ training samples. As such, the total number of
training samples is $20000$.

To evaluate the performance of respective algorithms, we adopt the
optimality gap $opg^k$ to measure the distance between the current solution
and the optimal solution. The optimality gap $opg^k$ is defined as
$opg^k\triangleq \frac{\|\boldsymbol{x}^k-\bar{\boldsymbol{W}}_{\infty}
\boldsymbol{x}^*\|_2}{\sqrt{N}}$,
where $\boldsymbol{x}^k\triangleq[\boldsymbol{x}_1^k;\cdots;
\boldsymbol{x}_N^k]$, $\bar{\boldsymbol{W}}_{\infty}\boldsymbol{x}^*
\triangleq[\boldsymbol{x}^*;\cdots;\boldsymbol{x}^*]$, and
$\boldsymbol{x}^*$ is the optimal solution obtained by solving
(\ref{simu-1}) in a centralized manner.

\subsection{Experimental Results}
First, we examine the performance of the proposed CFL-SAGA under
different choices of the triggering parameter $\rho$ as well as
different server topologies. Fig. \ref{fig2} $(a)$ plots the
optimality gap of CFL-SAGA vs. the number of iterations for
the random network. The triggering parameter $\rho$ varies from
$0$ to $50$. Clearly, $\rho=0$ corresponds to the case of
full-uploads, that is, all users are required to upload their
VR-SGs in each iteration. In general, we see that the convergence
speed of CFL-SAGA becomes slower as $\rho$ increases. This is
expected since a larger $\rho$ leads to a smaller number of user
uploads, resulting in a larger approximate error in the aggregated
gradient. Fig. \ref{fig2} $(b)$ and $(c)$ plot the optimality gap
of CFL-SAGA vs. the number of iterations for the ring graph
and the fully connected graph, respectively. For these two server
networks, the convergence behavior of CFL-SAGA is similar as
that in Fig. \ref{fig2} $(a)$. Not surprisingly, the algorithm
exhibits a faster convergence speed over a more well-connected
server network. Fig. \ref{fig2} $(d)$ plots the optimality gap of
CFL-SAGA vs. the communication overhead for the random
network. In particular, the communication overhead is measured by
the total number of VR-SGs that are uploaded to servers. It is
observed that to reach the same accuracy, the required number of
user uploads decreases as $\rho$ increases. Nevertheless, our
empirical results suggest that the highest communication
efficiency is achieved when $\rho=50$, and a larger value of
$\rho$ beyond $50$ does not yield further improvement on the
communication efficiency. It is also noticed that the CFL-SAGA
exhibits a significant advantage in terms of communication
efficiency as compared to the full-upload case ($\rho=0$). The
reason is that, in distributed optimization, the user gradient
usually changes slowly, especially in the high-precision regime.
Hence using the stale gradient to generate the aggregated gradient
often leads to a very small approximation error. As a result, even
with a very small number of user uploads, the algorithm can still
maintain a fast convergence speed.

\begin{figure*}[!htbp]
    \centering
    \subfigure[\scriptsize{Random graph.}]
    {\includegraphics[width=5cm,height=4cm]{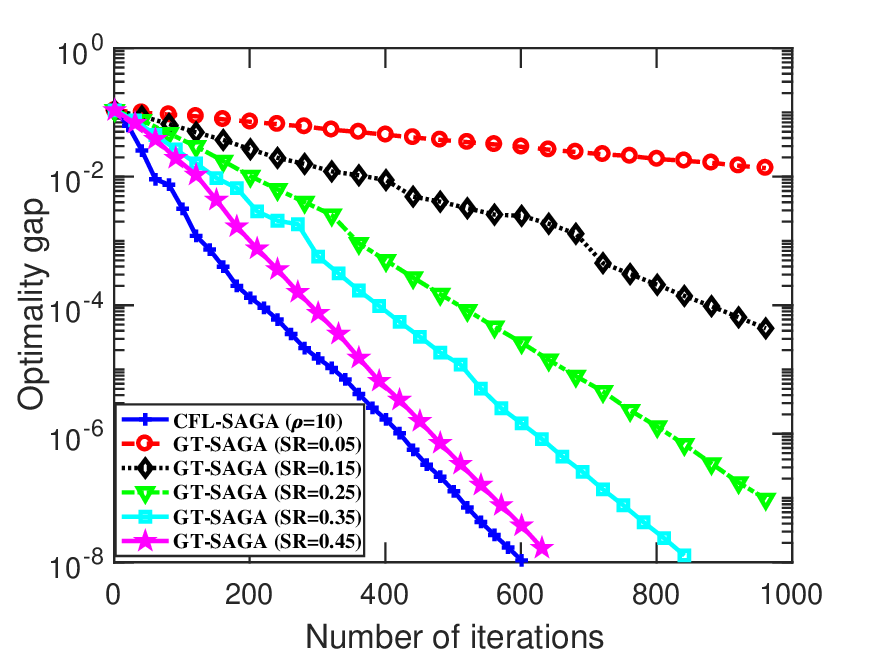}} \hfil
    \subfigure[\scriptsize{Ring graph.}]
    {\includegraphics[width=5cm,height=4cm]{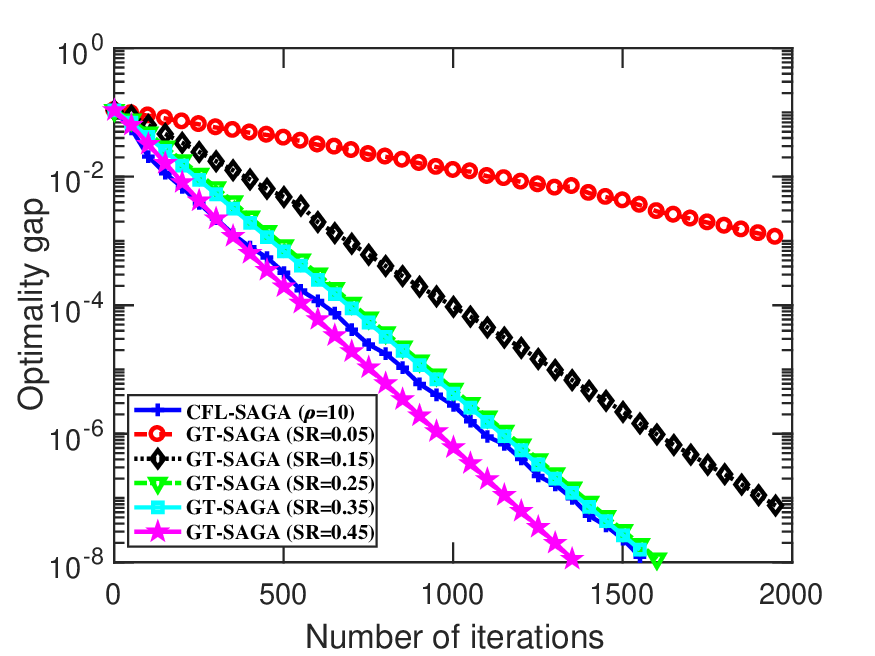}} \hfil
    \subfigure[\scriptsize{Fully connected graph.}]
    {\includegraphics[width=5cm,height=4cm]{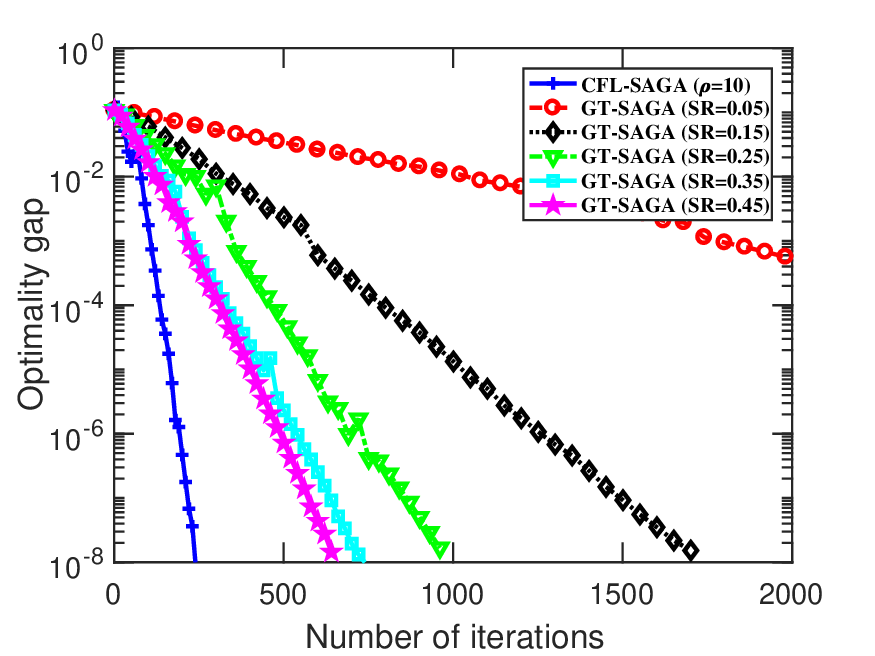}}
    \\
    \subfigure[\scriptsize{Random graph.}]
    {\includegraphics[width=5cm,height=4cm]{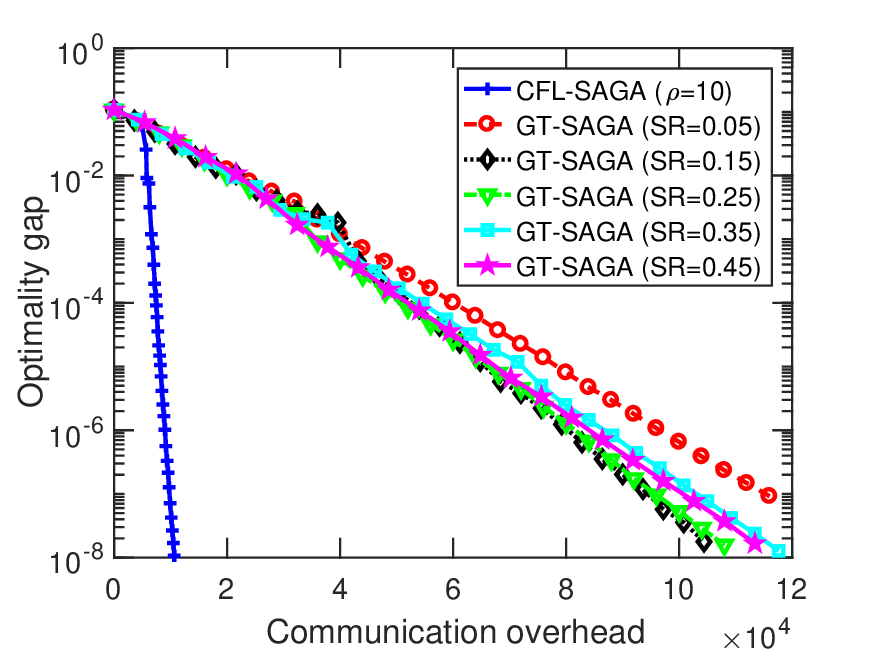}} \hfil
    \subfigure[\scriptsize{Ring graph.}]
    {\includegraphics[width=5cm,height=4cm]{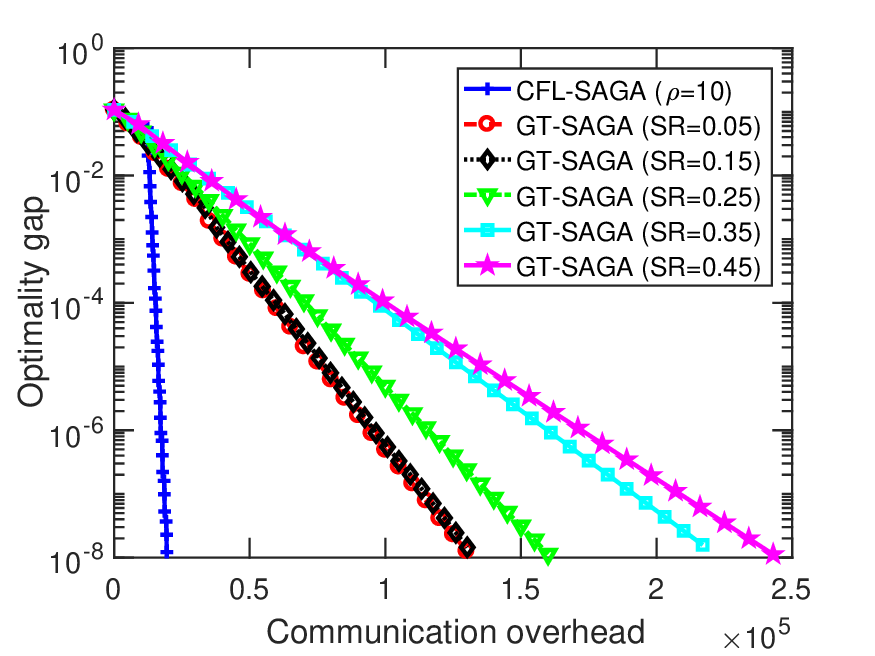}} \hfil
    \subfigure[\scriptsize{Fully connected graph.}]
    {\includegraphics[width=5cm,height=4cm]{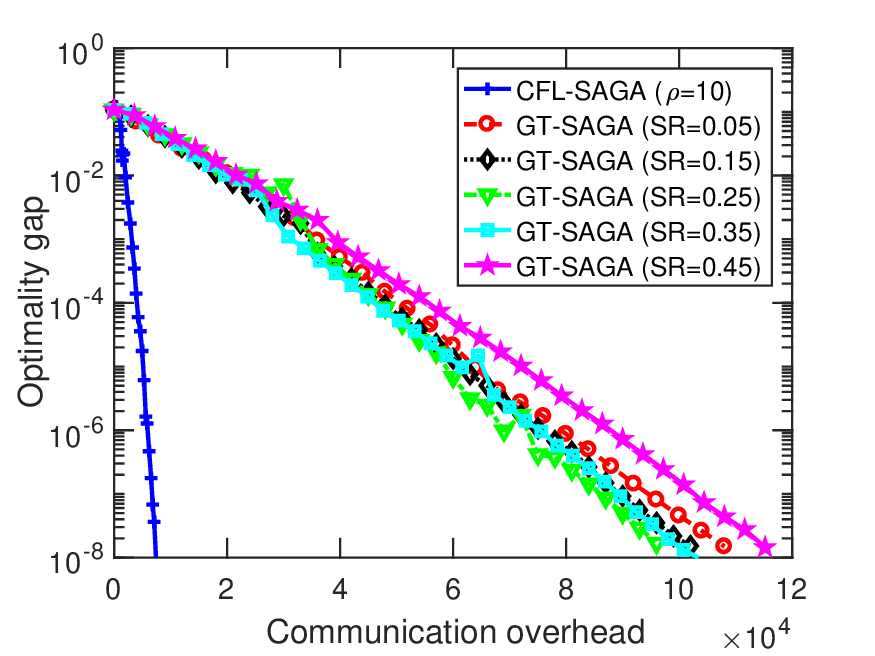}}
    \caption{Comparisons on $\ell_2$-regularized logistic regression.
    First row: Optimality gap vs. number of iterations on different
    server networks; Second row: Optimality gap vs. communication overhead
    on different server networks. SR is short for `Sampling rate'. }
    \label{fig3}
\end{figure*}

Next, we compare the performance of the proposed CFL-SAGA with
that of GT-SAGA, namely, Algorithm \ref{alg-3}. As shown in Fig.
\ref{fig2}, taking $\rho=10$ is sufficient to yield fast
convergence as well as high communication efficiency. We
thus fix $\rho=10$ for CFL-SAGA. Fig. \ref{fig3} $(a)$, $(b)$
and $(c)$ plot the optimality gap of respective algorithms vs. the
number of iterations for different networks. In Fig. \ref{fig3}
$(a)$, SR is an abbreviation for `sampling rate'. For instance,
$SR=0.15$ corresponds to the case that $0.15\times 20=3$ users are
selected by each server in each iteration. To make a full
comparison, the sampling rate of GT-SAGA is tuned from $0.05$ to
$0.45$. Clearly, the convergence speed of GT-SAGA becomes faster
as the sampling rate increases. It is observed that for the random
graph and the fully connected graph, the proposed CFL-SAGA is
faster than GT-SAGA that uses a sampling rate as large as $0.45$.
As for the communication overhead shown in Fig. \ref{fig3} $(d)$,
$(e)$ and $(f)$, we can see that the proposed CFL-SAGA exhibits
higher communication efficiency than GT-SAGA by orders of magnitude.
This advantage of CFL-SAGA is mainly due to the fact that the
proposed algorithm has the ability of uploading those most
informative gradients via the conditionally triggered user
selection mechanism, thus reducing the number of uploads substantially
without sacrificing the convergence speed of the proposed
algorithm. In fact, the averaged number of user uploads
per-iteration for our proposed algorithm is even smaller than that
of GT-SAGA with a sampling rate of $SR=0.05$.

To examine the computational complexity, we plot in Fig.
\ref{fig4} the average runtime of respective algorithms vs. the
number of iterations. We can see that GT-SAGA has a lower
per-iteration computational complexity than the CFL-SAGA. This is
because for CFL-SAGA, at each iteration each user is required to
compute its local stochastic gradient, whether or not this local
gradient is uploaded. As a comparison, the GT-SAGA only
requires those selected users to compute its stochastic gradient.
Note that the computational cost caused by the proposed CTUS
mechanism is usually negligible since the CTUS only involves very
simple calculations at each user with a complexity scaling
linearly with the dimension of the model variable
$\boldsymbol{x}$.

\begin{figure}
  \centering
  \includegraphics[width=5cm,height=4cm]{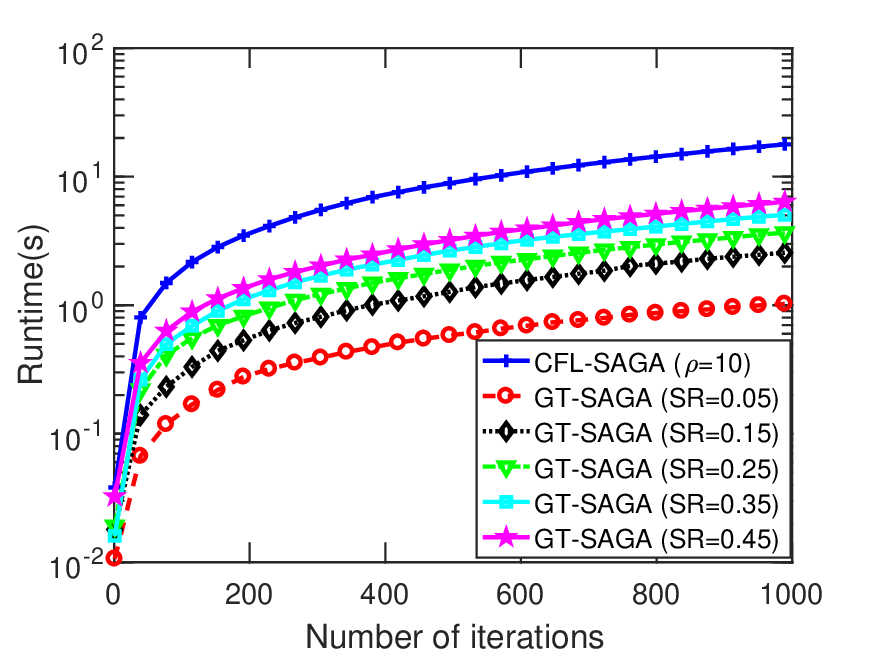}
  \caption{Runtime vs. number of iterations on the random graph}
  \label{fig4}
\end{figure}

\begin{figure}
  \centering
  \includegraphics[width=5cm,height=4cm]{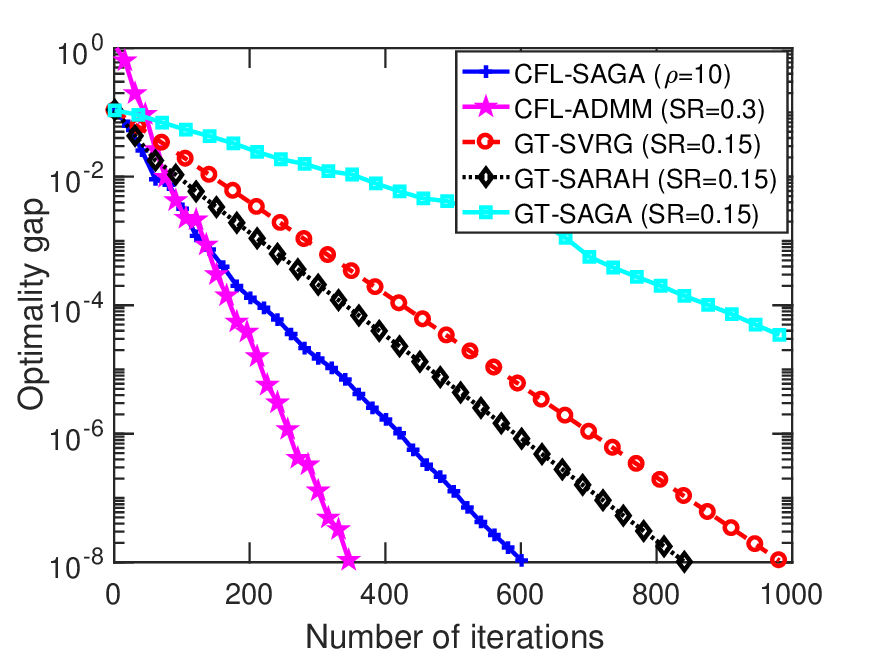}
  \caption{Optimality gap vs. number of iterations on the random graph}
  \label{fig5}
\end{figure}

\begin{figure}
  \centering
  \includegraphics[width=5cm,height=4cm]{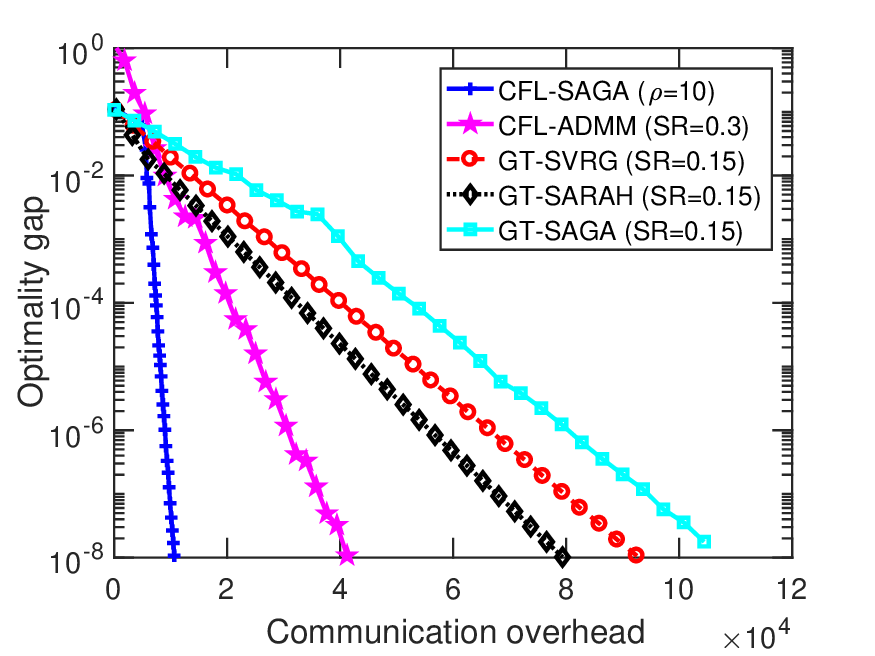}
  \caption{Optimality gap vs. communication overhead on the random graph}
  \label{fig6}
\end{figure}

At last, we compare the proposed CFL-SAGA with some
other state-of-the-art methods, namely, CFL-ADMM \cite{WangFang23},
GT-SVRG \cite{XinKhan20} and GT-SARAH \cite{XinKhan22}. Note that
CFL-ADMM randomly selects users to upload their local variables to
their respective servers. The user sampling rate of CFL-ADMM is
chosen as $0.3$ in our experiments, and $0.15$ for
GT-SVRG and GT-SARAG. Recall that both GT-SVRG and GT-SARAH are
double-loop based methods. Take GT-SVRG as an example, this
algorithms has an outer loop that aims to periodically update the
anchor gradient vector, a step very similar to the sum of
$\boldsymbol{g}_{ij}^{k+1}$. While the inner iteration is very
similar to each iteration of GT-SAGA. For this reason, both
GT-SVRG and GT-SARAH can be adapted to the CFL problem in a way
similar to GT-SAGA, except that they need to perform a full user
upload at the beginning of each outer loop. The parameters for
each algorithm is tuned to achieve the best communication
efficiency performance. Fig. \ref{fig5} and Fig. \ref{fig6}
respectively plot the optimality gap and the communication
overhead vs. the number of iterations on the random network.
Although the other algorithms achieve either similar or even faster
convergence speed compared to CFL-SAGA, the proposed CFL-SAGA
algorithm exhibits a significant advantage in terms of
communication efficiency. This is because the number of
per-iteration user uploads of CFL-SAGA is much smaller than those
in other algorithms.

\section{Conclusion}
\label{sec-9} In this paper, we proposed a SAGA-based
method for confederated learning. The proposed method employs
conditionally-triggered user selection (CTUS) to achieve
communication-efficient learning of the model vector. The major
innovation of the proposed method is the use of the CTUS
mechanism, which determines whether the user should upload its
local VR-SG by measuring its contribution relative to the progress
of the algorithm. Thanks to the CTUS mechanism, the proposed
algorithm only requires a very small number of uploads to maintain
fast convergence. Theoretical analysis indicates that the
proposed algorithm enjoys a fast linear convergence and
numerical results demonstrate the superior communication
efficiency of the proposed algorithm over GT-SAGA.

\appendices

\section{Preliminary Results}
First we list some inequalities that will be frequently
used in our analysis.
\begin{align}
&\textstyle\|\boldsymbol{x}+\boldsymbol{y}\|_2^2
\leq(1+p^{-1})\|\boldsymbol{x}\|_2^2+(1+p)\|\boldsymbol{y}\|_2^2,
\ \forall p>0,
\label{def-triangle-inequal}\\
&\textstyle\|\sum_{i=1}^N\boldsymbol{x}_i\|_2^2\leq N
\sum_{i=1}^N\|\boldsymbol{x}_i\|_2^2.
\label{def-multi-vect-inequal}\\
&\text{Variance decomposition}:
\nonumber\\
&\mathbb{E}\{\|\boldsymbol{x}-
\boldsymbol{y}\|_2^2\}=\mathbb{E}\{\|\boldsymbol{x}-
\mathbb{E}\{\boldsymbol{x}\}\|_2^2\}+\|\mathbb{E}\{\boldsymbol{x}\}-
\boldsymbol{y}\|_2^2, \ \forall \boldsymbol{y},
\label{def-VD}\\
&\textstyle\mathbb{E}\{\|\frac{1}{N}\sum_{i=1}^N
(x_i-\mathbb{E}\{x_i\})\|_2^2\}
=\frac{1}{N^2}\mathbb{E}\{\|x_i-\mathbb{E}\{x_i\}\|_2^2\},
\nonumber\\
&\text{with} \ \{x_i\}_{i=1}^{N} \ \text{independent of each other}.
\label{def-expectation-equal}
\end{align}

Next we present several intermediate results.
\newtheorem{lemma}{Lemma}
\begin{lemma}[Lemma 10 in \cite{QuLi17}]
\label{lemma1}
Suppose $f$ is $\mu$-strongly convex with its gradient
being $L$-Lipschitz continuous. Then for $\forall \boldsymbol{z}
\in\mathbb{R}^d$, it holds
\begin{align}
\textstyle\|\boldsymbol{z}-\alpha\nabla f(\boldsymbol{z})-\boldsymbol{z}^*\|_2
\leq (1-\mu\alpha)\|\boldsymbol{z}-\boldsymbol{z}^*\|_2
\end{align}
where $\boldsymbol{z}^*$ is the minimizer of $f$ and $\alpha\leq
\frac{1}{L}$ is a constant.
\end{lemma}

\newtheorem{lemma2}{Lemma}
\begin{lemma}[Lemma 2 in \cite{PuShi20}]
\label{lemma2}
Suppose $\boldsymbol{W}\in\mathbb{R}^{N\times N}$ is
primitive and doubly stochastic, then for $\forall \boldsymbol{z}\in
\mathbb{R}^{N}$, we have
\begin{align}
\textstyle\|\boldsymbol{W}\boldsymbol{z}-\boldsymbol{W}_{\infty}
\boldsymbol{z}\|_2
\leq \sigma\|\boldsymbol{z}-\boldsymbol{W}_{\infty}\boldsymbol{z}\|_2,
\end{align}
where $\sigma$ is the second largest singular value of $\boldsymbol{W}$.
\end{lemma}

\newtheorem{lemma3}{Lemma}
\begin{lemma}[Lemma 1 in \cite{XinKhan18}]
\label{lemma3}
Let the assumptions made in Section \ref{sec-obj-assumption} hold.
Then with $f(\bar{\boldsymbol{x}})=\frac{1}{N}\sum_{i=1}^N
f_i(\bar{\boldsymbol{x}})$ we have
\begin{align}
\textstyle\|\bar{\nabla} f(\boldsymbol{x})-\nabla f(\bar{\boldsymbol{x}})\|_2
\leq \frac{L}{\sqrt{N}}\|\boldsymbol{x}-
\bar{\boldsymbol{W}}_{\infty}\boldsymbol{x}\|_2.
\end{align}
\end{lemma}

\newtheorem{lemma4}{Lemma}
\begin{lemma}[Corollary 8.1.29 in \cite{HornJohnson12}]
\label{lemma6}
Suppose $\boldsymbol{T}\in\mathbb{R}^{n\times n}$ is a nonnegative
matrix and $\gamma>0$ is a constant. If there exists a positive
vector $\boldsymbol{\psi}\in\mathbb{R}^{n}$ such that
$\boldsymbol{T}\boldsymbol{\psi}\leq \gamma \boldsymbol{\psi}$,
then $\rho(\boldsymbol{T})\leq \gamma$, where
$\rho(\boldsymbol{T})\triangleq\max\{|\lambda_i|\}$ is the spectral
radius of $\boldsymbol{T}$.
\end{lemma}

At last, we mention that Step $1$ and $5$ in Algorithm \ref{alg-4}
can be compactly written as
\begin{align}
&\textstyle\boldsymbol{x}^{k+1}=\bar{\boldsymbol{W}}\boldsymbol{x}^{k}-
\alpha\boldsymbol{y}^k,
\nonumber\\
&\textstyle\boldsymbol{y}^{k+1}=\bar{\boldsymbol{W}}\boldsymbol{y}^{k}
+\boldsymbol{g}^{k+1}-\boldsymbol{g}^k,
\label{PropAlg-compact-form}
\end{align}
where $\boldsymbol{g}^k\triangleq[\boldsymbol{g}_1^k;\cdots;
\boldsymbol{g}_N^k]$.

The left hand side of (\ref{theorem2-content-3}) is the average
of $\mathbb{E}_{t_{ij}^{k+1},t_{ij}^{k}}
\{\|\boldsymbol{\Delta}_{ij}^{k+1}\|_2^2\}$s, while $\frac{1}{N}
\sum_{i=1}^N\|\sum_{i'=1}^{N}w_{ii'}\boldsymbol{x}_{i'}^{k+1}-
\boldsymbol{x}_i^{k+1}\|_2$ in the right hand side of
(\ref{theorem2-content-3}) is the average of
$\|\sum_{i'=1}^{N}w_{ii'}\boldsymbol{x}_{i'}^{k+1}-
\boldsymbol{x}_i^{k+1}\|_2^2$s. The inequality (\ref{theorem2-content-3})
indicates that
\begin{align}
\textstyle \mathbb{E}_{t_{ij}^{k+1},t_{ij}^{k}}
\{\|\boldsymbol{\Delta}_{ij}^{k+1}\|_2^2\}\lesssim
C_3\|\sum_{i'=1}^{N}w_{ii'}\boldsymbol{x}_{i'}^{k+1}-
\boldsymbol{x}_i^{k+1}\|_2^2
\end{align}
holds in the average sense. As such, if the value of $\rho$ is
set to be sufficiently large, then (\ref{main-alg-4}) will not
be triggered for most of the users. Consequently, the number
of uploads can be significantly reduced.

\section{Proof of Proposition \ref{Proposition-1}}
\label{appendix-proposi}
Without loss of generality, we assume that $t_{ij}^{k+1}\neq
t_{ij}^{k}$. Recall that $\boldsymbol{\Delta}_{ij}^{k+1}$ in
(\ref{discussion-2-2}) is defined as $\boldsymbol{\Delta}_{ij}^{k+1}
=\boldsymbol{g}_{ij}^{k+1}-\boldsymbol{g}_{ij}^k$. Combining this
with the definition of $\boldsymbol{g}_{ij}^{k+1}$ we have
\begin{align}
&\boldsymbol{\Delta}_{ij}^{k+1}
=S_{ij}\cdot\Big(\nabla f_{ij,t_{ij}^{k+1}}
(\boldsymbol{x}_{i}^{k+1})-\nabla f_{ij,t_{ij}^{k+1}}
(\boldsymbol{\phi}_{ij,t_{ij}^{k+1}}^{k})\Big)
\nonumber\\
&+\nabla f_{ij,t_{ij}^k}(\boldsymbol{x}_{i}^{k})
-S_{ij}\cdot\Big(\nabla f_{ij,t_{ij}^{k}}
(\boldsymbol{x}_{i}^{k})-\nabla f_{ij,t_{ij}^{k}}
(\boldsymbol{\phi}_{ij,t_{ij}^{k}}^{k})\Big)
\nonumber\\
&-\nabla f_{ij,t_{ij}^{k}}(\boldsymbol{\phi}_{ij,t_{ij}^{k}}^{k})
\nonumber\\
&=S_{ij}\cdot\Big(\nabla f_{ij,t_{ij}^{k+1}}
(\boldsymbol{x}_{i}^{k+1})-\nabla f_{ij,t_{ij}^{k+1}}
(\boldsymbol{\phi}_{ij,t_{ij}^{k+1}}^{k})\Big)
\nonumber\\
&-(S_{ij}-1)\cdot\Big(\nabla f_{ij,t_{ij}^{k}}
(\boldsymbol{x}_{i}^{k})-\nabla f_{ij,t_{ij}^{k}}
(\boldsymbol{\phi}_{ij,t_{ij}^{k}}^{k})\Big)
\nonumber\\
&\overset{(a)}{=}S_{ij}\cdot\Big(\nabla f_{ij,t_{ij}^{k+1}}
(\boldsymbol{\phi}_{ij,t_{ij}^{k+1}}^{k+1})-\nabla f_{ij,t_{ij}^{k+1}}
(\boldsymbol{\phi}_{ij,t_{ij}^{k+1}}^{k})\Big)
\nonumber\\
&-(S_{ij}-1)\cdot\Big(\nabla f_{ij,t_{ij}^{k}}
(\boldsymbol{\phi}_{ij,t_{ij}^{k}}^{k+1})-\nabla f_{ij,t_{ij}^{k}}
(\boldsymbol{\phi}_{ij,t_{ij}^{k}}^{k})\Big),
\label{comm-effici-1}
\end{align}
where in $(a)$ we have used the fact that
$\boldsymbol{\phi}_{ij,t_{ij}^{k+1}}^{k+1}=\boldsymbol{x}_i^{k+1}$
(resp. $\boldsymbol{\phi}_{ij,t_{ij}^{k}}^{k+1}=
\boldsymbol{\phi}_{ij,t_{ij}^{k}}^{k}=\boldsymbol{x}_{i}^{k}$).
Combing (\ref{comm-effici-1}) with the Lipschitz continuity
of $f_{ij,t_{ij}^{k}}$ and $f_{ij,t_{ij}^{k+1}}$ we have
\begin{align}
&\|\boldsymbol{\Delta}_{ij}^{k+1}\|_2^2
\leq 2S_{ij}^2\cdot L_{ij,t_{ij}^{k+1}}^2\cdot
\Big\|\boldsymbol{\phi}_{ij,t_{ij}^{k+1}}^{k+1}
-\boldsymbol{\phi}_{ij,t_{ij}^{k+1}}^{k}\Big\|_2^2
\nonumber\\
&+2S_{ij}^2\cdot L_{ij,t_{ij}^{k}}^2\cdot
\Big\|\boldsymbol{\phi}_{ij,t_{ij}^{k}}^{k+1}
-\boldsymbol{\phi}_{ij,t_{ij}^{k}}^{k}\Big\|_2^2
\nonumber\\
&\leq 4S_{ij}^2\cdot L_{ij,t_{ij}^{k+1}}^2\cdot
\Big(\big\|\boldsymbol{\phi}_{ij,t_{ij}^{k+1}}^{k+1}
-\boldsymbol{x}^*\big\|_2^2+
\big\|\boldsymbol{\phi}_{ij,t_{ij}^{k+1}}^{k}-
\boldsymbol{x}^*\big\|_2^2\Big)
\nonumber\\
&+4S_{ij}^2\cdot L_{ij,t_{ij}^{k}}^2\cdot
\Big(\big\|\boldsymbol{\phi}_{ij,t_{ij}^{k}}^{k+1}
-\boldsymbol{x}^*\big\|_2^2+
\big\|\boldsymbol{\phi}_{ij,t_{ij}^{k}}^{k}-
\boldsymbol{x}^*\big\|_2^2\Big).
\label{comm-effici-2}
\end{align}
where $L_{ij,t}$ is the Lipschitz constant corresponding to
$\nabla f_{ij,t}$. From (\ref{comm-effici-2}) we can further
deduce that
\begin{align}
&\textstyle
\sum\limits_{i=1}^N\sum\limits_{j=1}^{P_i}
\|\boldsymbol{\Delta}_{ij}^{k+1}\|_2^2\leq
 \sum\limits_{i=1}^N\sum\limits_{j=1}^{P_i}
4S_{ij}^2\Big(L_{ij,t_{ij}^{k+1}}^2
\big\|\boldsymbol{\phi}_{ij,t_{ij}^{k+1}}^{k+1}
-\boldsymbol{x}^*\big\|_2^2+
\nonumber\\
& \textstyle L_{ij,t_{ij}^{k}}^2
\big\|\boldsymbol{\phi}_{ij,t_{ij}^{k}}^{k+1}-
\boldsymbol{x}^*\big\|_2^2\Big)+
\sum\limits_{i=1}^N\sum\limits_{j=1}^{P_i}
4S_{ij}^2\Big(L_{ij,t_{ij}^{k+1}}^2
\big\|\boldsymbol{\phi}_{ij,t_{ij}^{k+1}}^{k}
\nonumber\\
&\textstyle -\boldsymbol{x}^*\big\|_2^2+L_{ij,t_{ij}^{k}}^2
\big\|\boldsymbol{\phi}_{ij,t_{ij}^{k}}^{k}-
\boldsymbol{x}^*\big\|_2^2\Big).
\label{comm-effici-3}
\end{align}
which implies that
\begin{align}
&\textstyle\mathbb{E}_{t_{ij}^{k+1},t_{ij}^{k}}
\{\sum_{i=1}^N\sum_{j=1}^{P_i}
\|\boldsymbol{\Delta}_{ij}^{k+1}\|_2^2\}
\nonumber\\
&\overset{(a)}{\leq}\textstyle
2\sum\limits_{i=1}^N\sum\limits_{j=1}^{P_i}
4S_{ij}^2\sum\limits_{t=1}^{S_{ij}} \frac{L_{ij,t}^2}{S_{ij}}
\Big(\big\|\boldsymbol{\phi}_{ij,t}^{k+1}
-\boldsymbol{x}^*\big\|_2^2+
\big\|\boldsymbol{\phi}_{ij,t}^{k}
-\boldsymbol{x}^*\big\|_2^2\Big)
\nonumber\\
&\leq 2 \bar{L} (D^{k+1}+D^k)
\label{comm-effici-3-1}
\end{align}
where
\begin{align}
&\bar{L}=\text{max}\{L_{k+1},L_{k}\},
\nonumber\\
&\textstyle L_{k+1}=\Big(\sum_{i=1}^N\sum_{j=1}^{P_i}
4S_{ij}^2\sum_{t=1}^{S_{ij}} \frac{L_{ij,t}^2}{S_{ij}}
\big\|\boldsymbol{\phi}_{ij,t}^{k+1}
-\boldsymbol{x}^*\big\|_2^2\Big)\Big/
\nonumber\\
&\qquad\quad\textstyle \Big(\underbrace{\textstyle
\sum_{i=1}^N\sum_{j=1}^{P_i}\sum_{t=1}^{S_{ij}}
\big\|\boldsymbol{\phi}_{ij,t}^{k+1}
-\boldsymbol{x}^*\big\|_2^2}_{=D^{k+1}}\Big),
\nonumber\\
&\textstyle L_k=\Big(\sum_{i=1}^N\sum_{j=1}^{P_i}
4S_{ij}^2\sum_{t=1}^{S_{ij}} \frac{L_{ij,t}^2}{S_{ij}}
\big\|\boldsymbol{\phi}_{ij,t}^{k}
-\boldsymbol{x}^*\big\|_2^2\Big)\Big/
\nonumber\\
&\qquad\quad\textstyle \Big(\underbrace{\textstyle
\sum_{i=1}^N\sum_{j=1}^{P_i}\sum_{t=1}^{S_{ij}}
\big\|\boldsymbol{\phi}_{ij,t}^{k}
-\boldsymbol{x}^*\big\|_2^2}_{D^{k}}\Big).
\label{comm-effici-3-3}
\end{align}
Taking the full expectation for both sides of (\ref{comm-effici-3-1})
yields
\begin{align}
&\textstyle\mathbb{E}\{\sum_{i=1}^N\sum_{j=1}^{P_i}
\|\boldsymbol{\Delta}_{ij}^{k+1}\|_2^2\}\leq
\mathbb{E}\{2 \bar{L} (D^{k+1}+D^k)\}
\nonumber\\
&\textstyle \approx  \mathbb{E}\{4\bar{L}D^{k+1}\}
\nonumber\\
&\overset{(\ref{lemma5-1-content})}{\leq}\textstyle
\mathbb{E}\Big\{4 \bar{L}\Big( (1-\frac{1}{\overline{S_{ij}}}) D^{k}
+ 2\overline{P_i}X^{k+1}+2\overline{P_i}N\bar{X}^{k+1}\Big)\Big\}
\label{comm-effici-3-2}
\end{align}
Assuming $\mathbb{E}\{D^{k+1}\}\approx \mathbb{E}\{D^{k}\}$,
the above inequality implies that
\begin{align}
&\textstyle \mathbb{E}\{4 \bar{L} D^{k+1}\}
\leq \mathbb{E}\Big\{8 \overline{S_{ij}}\bar{L}\overline{P_i}
\big(X^{k+1}+N\bar{X}^{k+1}\big)\Big\}
\nonumber\\
&=\mathbb{E}\Big\{8\overline{S_{ij}} \bar{L}\overline{P_i}
\big(\|\boldsymbol{x}^{k+1}-
\bar{\boldsymbol{W}}_{\infty}\boldsymbol{x}^{k+1}\|_2^2
+\|\bar{\boldsymbol{W}}_{\infty}\boldsymbol{x}^{k+1}-
\tilde{\boldsymbol{x}}^*\|_2^2\big)\Big\}
\nonumber\\
 &\overset{(\ref{theorem2-content-1})}{\leq}
 \mathbb{E}\Big\{8(1+C_1)\overline{S_{ij}}
 \bar{L}\overline{P_i}\cdot\|\boldsymbol{x}^{k+1}-
\bar{\boldsymbol{W}}_{\infty}\boldsymbol{x}^{k+1}\|_2^2\Big\}
\nonumber\\
&=  \mathbb{E}\Big\{8(1+C_1)\overline{S_{ij}}\bar{L}\overline{P_i}
\cdot\|\boldsymbol{x}^{k+1}-\bar{\boldsymbol{W}}\boldsymbol{x}^{k+1}
+\bar{\boldsymbol{W}}\boldsymbol{x}^{k+1}-
\nonumber\\
&\qquad\qquad\qquad\qquad\qquad\qquad\qquad\qquad
\bar{\boldsymbol{W}}_{\infty}\boldsymbol{x}^{k+1}\|_2^2\Big\}
\nonumber\\
&\leq \mathbb{E}\Big\{16(1+C_1)\overline{S_{ij}}\bar{L}
\overline{P_i}\big(\|\boldsymbol{x}^{k+1}-
\bar{\boldsymbol{W}}\boldsymbol{x}^{k+1}\|_2^2
\nonumber\\
&\qquad\qquad\qquad\qquad+
\|\bar{\boldsymbol{W}}\boldsymbol{x}^{k+1}-
\bar{\boldsymbol{W}}_{\infty}\boldsymbol{x}^{k+1}\|_2^2\big)\Big\}
\nonumber\\
&\overset{(\ref{theorem2-content-2})}{\leq}
\mathbb{E}\Big\{\underbrace{16(1+C_1)(1+C_2)
\overline{S_{ij}}\bar{L}\overline{P_i}}_{[(\ref{comm-effici-4})-1]}
\cdot\|\boldsymbol{x}^{k+1}-
\bar{\boldsymbol{W}}\boldsymbol{x}^{k+1}\|_2^2\Big\}
\nonumber\\
&=  \textstyle \mathbb{E}\Big\{[(\ref{comm-effici-4})-1]\cdot
\sum\limits_{i=1}^N\big\|\sum_{i'=1}^{N}w_{ii'}
\boldsymbol{x}_{i'}^{k+1}-\boldsymbol{x}_i^{k+1}\big\|_2^2\Big\}.
\label{comm-effici-4}
\end{align}
Combining (\ref{comm-effici-3-2}) and (\ref{comm-effici-4})
we obtain
\begin{align}
&\textstyle\mathbb{E}\{\sum_{i=1}^N\sum_{j=1}^{P_i}
\|\boldsymbol{\Delta}_{ij}^{k+1}\|_2^2\}
\overset{(\ref{comm-effici-3-2})}{\lesssim}
 \mathbb{E}\{4\bar{L}D^{k+1}\}
\nonumber\\
&\textstyle \overset{(\ref{comm-effici-4})}{\leq}
\mathbb{E}\Big\{[(\ref{comm-effici-4})-1]\cdot
\sum\limits_{i=1}^N\big\|\sum_{i'=1}^{N}w_{ii'}
\boldsymbol{x}_{i'}^{k+1}-\boldsymbol{x}_i^{k+1}\big\|_2^2\Big\}.
\label{comm-effici-5}
\end{align}
Multiplying $\frac{1}{\sum_{i=1}^N P_i}$ to both sides of
(\ref{comm-effici-5}) yields the desired result.

\section{Proving the First Inequality in (\ref{proofn-2})}
The first inequality, i.e., (\ref{lemma10-content}), is proved in
the following lemma. For completeness and clarity, we provide a
simple proof of this lemma.
\newtheorem{lemma5}{Lemma}
\begin{lemma}[Lemma 4 in \cite{XinKhan20}]
\label{lemma10}
Let the assumptions made in Section \ref{sec-obj-assumption} hold.
Then
\begin{align}
&\textstyle X^{k+1}\leq\frac{1+\sigma^2}{2}\cdot X^k+
\frac{2\alpha^2}{1-\sigma^2}\cdot Y^k,
\label{lemma10-content}\\
&\textstyle X^{k+1}\leq2\sigma^2\cdot X^{k}+2\alpha^2\cdot Y^k.
\label{lemma10-content2}
\end{align}
\end{lemma}
\begin{proof}
According to the first equality of (\ref{PropAlg-compact-form}),
it holds that
\begin{align}
&\textstyle X^{k+1}\overset{(\ref{PropAlg-compact-form})}{=}
\|\bar{\boldsymbol{W}}\boldsymbol{x}^{k}-\alpha\boldsymbol{y}^k-
\bar{\boldsymbol{W}}_{\infty}(\bar{\boldsymbol{W}}\boldsymbol{x}^{k}-
\alpha\boldsymbol{y}^k)\|_2^2
\nonumber\\
\overset{(a)}{\leq}&
\textstyle(1+p)\|\bar{\boldsymbol{W}}\boldsymbol{x}^{k}
-\bar{\boldsymbol{W}}_{\infty}\boldsymbol{x}^{k}\|_2^2
+(1+\frac{1}{p})\alpha^2\|\boldsymbol{y}^k
-\bar{\boldsymbol{W}}_{\infty}\boldsymbol{y}^k\|_2^2
\nonumber\\
\overset{(b)}{\leq}&
\textstyle (1+p)\sigma^2\cdot X^k
+(1+\frac{1}{p})\alpha^2\cdot Y^k, \ \forall p>0,
\end{align}
where $(a)$ is because $\bar{\boldsymbol{W}}_{\infty}\bar{\boldsymbol{W}}
=\bar{\boldsymbol{W}}_{\infty}$ as well as (\ref{def-triangle-inequal}),
and $(b)$ comes from Lemma \ref{lemma2} (note that $\boldsymbol{W}$ and
$\bar{\boldsymbol{W}}$ shares similar properties because $\bar{\boldsymbol{W}}
\triangleq\boldsymbol{W}\otimes\boldsymbol{I}_{d}$). Setting
$p=\frac{1-\sigma^2}{2\sigma^2}$ (resp. $1$) and using $\sigma<1$ leads
to (\ref{lemma10-content}) (resp. (\ref{lemma10-content2})).
\end{proof}

\section{Proving the Second Inequality in (\ref{proofn-2})}
\label{appendix-B}
Before starting, we introduce several notations that will be used
later. let $\mathcal{F}^{k}=\{t_{ij}^{k'}\}_{i,j,k'=1}^{i=N,j=P_i,k'=k}$
denote the set of random variables appeared before the $(k+1)$th
iteration. Also let $\mathbb{E}^{\mathcal{F}^{k}}\{\cdot\}$
represent the conditional expectation that is conditioned on
$\mathcal{F}^{k}$. At last, we use $r^k$ to refer to
$\{t_{ij}^{k}\}_{i=1,j=1}^{i=N,j=P_i}$, which is the set of random
variables in the $k$th iteration.
The second inequality in (\ref{proofn-2}) reads as follows
\begin{align}
\textstyle\mathbb{E}\{\bar{X}^{k+1}\}
\leq \mathbb{E}\{b_1\cdot\bar{X}^{k}+b_2\cdot X^k+
b_3\cdot D^{k-1}\}
\label{lemma4-1-1-content}
\end{align}
We will prove this inequality in Lemma \ref{lemma5}. Before
presenting Lemma \ref{lemma5}, we first provide an important
intermediate result.

\newtheorem{lemma4-1}{Lemma}
\begin{lemma}
\label{lemma4-1}
Suppose the assumptions in Section \ref{sec-obj-assumption}
hold. Let $(\boldsymbol{x}^k,\boldsymbol{y}^k)$ be generated by
the proposed Algorithm \ref{alg-4}. Then
\begin{align}
\textstyle\mathbb{E}
\{\|\boldsymbol{g}^k-\nabla f(\boldsymbol{x}^k)\|_2^2\}
\leq \textstyle\mathbb{E}\{c_1 X^k+c_2\bar{X}^k+c_3 D^{k-1}\},
\label{lemma4-1-content}
\end{align}
where $\nabla f(\boldsymbol{x}^k)\triangleq [\nabla f_1(\boldsymbol{x}_1^k);
\cdots;\nabla f_N(\boldsymbol{x}_N^k)]$, $\boldsymbol{g}^k\triangleq
[\boldsymbol{g}_1^k;\cdots;\boldsymbol{g}_N^k]$, $\boldsymbol{g}_i^k$
is defined in Algorithm \ref{alg-4},
\begin{align}
&\textstyle c_1\triangleq 8L^2(1+\overline{P_i}\cdot\overline{S_{ij}}^2)
+12\rho(1+\sigma^2)\overline{P_i}^2,
\nonumber\\
&\textstyle c_2\triangleq 8L^2(1+\overline{P_i}\cdot\overline{S_{ij}}^2)N, \
c_3\triangleq 4L^2\overline{S_{ij}},
\label{lemma4-1-content2}
\end{align}
$\overline{P_i}\triangleq\max_i\{P_i\}$ and $\overline{S_{ij}}=
\max_{i,j}\{S_{ij}\}$.
\end{lemma}
\begin{proof}
For notational convenience, denote $G^k=\|\boldsymbol{g}^k-
\nabla f(\boldsymbol{x}^k)\|_2^2$ and $G_i^k=\|\boldsymbol{g}_i^k-
\nabla f_i(\boldsymbol{x}^k)\|_2^2$. Since $\mathbb{E}\{G^k\}=
\mathbb{E}\{\sum_{i=1}^N G_i^k\}$, we should separately bound each
$\mathbb{E}\{G_i^k\}$. For this term, we first consider its
conditional expectation:
\begin{align}
&\textstyle\mathbb{E}_{r^{k}}^{\mathcal{F}^{k-1}}\{G_i^k\}
\overset{(a)}{=}
\mathbb{E}_{r^{k}}^{\mathcal{F}^{k-1}}
\Big\{\Big\|\boldsymbol{g}_i^k-\mathbb{E}_{r^{k}}^{\mathcal{F}^{k-1}}
\{\boldsymbol{g}_i^k+\sum\limits_{j\in\mathcal{N}_{i}^{k}}
\boldsymbol{\Delta}_{ij}^{k}\}\Big\|_2^2\Big\}
\nonumber\\
&\textstyle\overset{(\ref{def-multi-vect-inequal})}{\leq}
2\mathbb{E}_{r^{k}}^{\mathcal{F}^{k-1}}
\Big\{\|\boldsymbol{g}_i^k-\mathbb{E}_{r^{k}}^{\mathcal{F}^{k-1}}
\{\boldsymbol{g}_i^k\}\|_2^2
+\big\|\mathbb{E}_{r^{k}}^{\mathcal{F}^{k-1}}
\big\{\sum\limits_{j\in\mathcal{N}_{i}^{k}}
\boldsymbol{\Delta}_{ij}^{k}\big\}\big\|_2^2\Big\}
\nonumber\\
&\textstyle\overset{(b)}{\leq}
2\mathbb{E}_{r^{k}}^{\mathcal{F}^{k-1}}
\{\|\boldsymbol{g}_i^k-\nabla f_i(\boldsymbol{x}^*)\|_2^2
+\textstyle\|\sum_{j\in\mathcal{N}_{i}^{k}}
\boldsymbol{\Delta}_{ij}^{k}\|_2^2\},
\label{lemmaX-1}
\end{align}
where $(a)$ follows from
\begin{align}
&\textstyle\mathbb{E}_{r^{k}}^{\mathcal{F}^{k-1}}
\{\boldsymbol{g}_i^k+\sum_{j\in\mathcal{N}_{i}^{k}}
\boldsymbol{\Delta}_{ij}^{k}\}
\nonumber\\
&\textstyle\overset{(\ref{main-alg-3})}{=}\mathbb{E}_{r^{k}}^{\mathcal{F}^{k-1}}
\{\textstyle\sum_{j=1}^{P_i}(\boldsymbol{g}_{ij}^{k-1}
+\boldsymbol{\Delta}_{ij}^{k})\}
=\textstyle\nabla f_i(\boldsymbol{x}^k),
\label{lemmaX-1-1}
\end{align}
$(b)$ has invoked (\ref{def-VD}) for the first term (by treating
$\nabla f_i(\boldsymbol{x}^*)$ as $\boldsymbol{y}$) and
Jensen's inequality for the second term. Regarding
$\|\sum_{j\in\mathcal{N}_{i}^{k}}
\boldsymbol{\Delta}_{ij}^{k}\|_2^2$ in the last line of (\ref{lemmaX-1}),
we have
\begin{align}
&\mathbb{E}_{r^{k}}^{\mathcal{F}^{k-1}}
\{\textstyle \|\sum_{j\in\mathcal{N}_{i}^{k}}
\boldsymbol{\Delta}_{ij}^{k}\|_2^2\}
\overset{(\ref{def-multi-vect-inequal})}{\leq}
\mathbb{E}_{r^{k}}^{\mathcal{F}^{k-1}}
\{|\mathcal{N}_{i}^{k}|\sum_{j\in\mathcal{N}_{i}^{k}}
\|\boldsymbol{\Delta}_{ij}^{k}\|_2^2\}
\nonumber\\
&\textstyle\overset{(a)}{\leq} |\mathcal{N}_{i}^{k}|^2
\rho\|\boldsymbol{x}_i^{k}-
[\bar{\boldsymbol{W}}]_{i}\boldsymbol{x}^{k}\|_2^2,
\label{lemmaX-2}
\end{align}
where $[\bar{\boldsymbol{W}}]_{i}$ represents the $i$th block-row
of $\bar{\boldsymbol{W}}$, and $(a)$ comes from Step 3.2 of Algorithm
\ref{alg-4}. Substituting (\ref{lemmaX-2}) into (\ref{lemmaX-1})
yields
\begin{align}
&\textstyle\mathbb{E}_{r^{k}}^{\mathcal{F}^{k-1}}
\{G_i^k\}
\nonumber\\
&\textstyle
\leq \mathbb{E}_{r^{k}}^{\mathcal{F}^{k-1}}
\{2\|\boldsymbol{g}_i^k-\nabla f_i(\boldsymbol{x}^*)\|_2^2\}
+2|\mathcal{N}_{i}^{k}|^2\rho\|\boldsymbol{x}_i^{k}-
[\bar{\boldsymbol{W}}]_{i}\boldsymbol{x}^{k}\|_2^2
\nonumber\\
&\overset{(\ref{def-VD})}{=}\textstyle
2|\mathcal{N}_{i}^{k}|^2\rho\|\boldsymbol{x}_i^{k}-
[\bar{\boldsymbol{W}}]_{i}\boldsymbol{x}^{k}\|_2^2+
2\underbrace{\textstyle\mathbb{E}_{r^k}^{\mathcal{F}^{k-1}}
\{\|\boldsymbol{g}_i^k-\mathbb{E}_{r^k}^{\mathcal{F}^{k-1}}
\{\boldsymbol{g}_i^k)\|_2^2\}}_{\triangleq T_2}
\nonumber\\
&\textstyle+2\underbrace{\textstyle\|\mathbb{E}_{r^k}^{\mathcal{F}^{k-1}}
\{\boldsymbol{g}_i^k\}-
\nabla f_i(\boldsymbol{x}^*)\|_2^2}_{\triangleq T_1}.
\label{lemmaX-3}
\end{align}
We next separately bound $T_1$ and $T_2$. For $T_1$ we have
\begin{align}
&T_1=\textstyle \|\mathbb{E}_{r^{k}}^{\mathcal{F}^{k-1}}
\{\boldsymbol{g}_i^k+
\sum_{j\in\mathcal{N}_{i}^{k}}(\boldsymbol{\Delta}_{ij}^{k}
-\boldsymbol{\Delta}_{ij}^{k})\}-\nabla f_i(\boldsymbol{x}^*)\|_2^2
\nonumber\\
&\overset{(a)}{\leq}\textstyle 2\|\nabla f_i(\boldsymbol{x}_i^k)
-\nabla f_i(\boldsymbol{x}^*)\|_2^2+
2\|\mathbb{E}_{r^{k}}^{\mathcal{F}^{k-1}}\{
\sum_{j\in\mathcal{N}_{i}^{k}}\boldsymbol{\Delta}_{ij}^{k}\}\|_2^2
\nonumber\\
&\overset{(b)}{\leq}\textstyle 2L^2\|\boldsymbol{x}_i^k-\boldsymbol{x}^*\|_2^2+
2\mathbb{E}_{r^{k}}^{\mathcal{F}^{k-1}}\{|\mathcal{N}_{i}^{k}|
\sum_{j\in\mathcal{N}_{i}^{k}}
\|\boldsymbol{\Delta}_{ij}^{k}\|_2^2\}
\nonumber\\
&\overset{(\ref{lemmaX-2})}{\leq}\textstyle 2L^2\|\boldsymbol{x}_i^k-\boldsymbol{x}^*\|_2^2+
2|\mathcal{N}_{i}^{k}|^2\rho\|\boldsymbol{x}_i^{k}-
[\bar{\boldsymbol{W}}]_{i}\boldsymbol{x}^{k}\|_2^2,
\label{lemmaX-5}
\end{align}
where $(a)$ has used (\ref{lemmaX-1-1}) and (\ref{def-triangle-inequal}),
$(b)$ has invoked the Lipschitz continuity of $\nabla f_i$ and
Jensen's inequality. For $T_2$ it holds that
\begin{align}
&T_2\overset{(\ref{main-alg-3})}{=}
\textstyle\mathbb{E}_{r^{k}}^{\mathcal{F}^{k-1}}
\{\|\sum_{j\notin\mathcal{N}_{i}^{k}}
\boldsymbol{g}_{ij}^{k}-\mathbb{E}_{r^{k}}^{\mathcal{F}^{k-1}}
\{\sum_{j\notin\mathcal{N}_{i}^{k}}
\boldsymbol{g}_{ij}^{k}\}\|_2^2\}
\nonumber\\
&\overset{(\ref{def-expectation-equal})}{=}
\textstyle\mathbb{E}_{r^{k}}^{\mathcal{F}^{k-1}}
\{\sum_{j\notin\mathcal{N}_{i}^{k}}
\|\boldsymbol{g}_{ij}^k-\mathbb{E}_{r^{k}}^{\mathcal{F}^{k-1}}
\{\boldsymbol{g}_{ij}^k\}\|_2^2\},
\nonumber\\
&\leq \textstyle \mathbb{E}_{\{t_{ij}^{k}\}}^{\mathcal{F}^{k-1}}
\{\textstyle\sum_{j=1}^{P_i}
\|\boldsymbol{g}_{ij}^k-\mathbb{E}_{t_{ij}^k}^{\mathcal{F}^{k-1}}
\{\boldsymbol{g}_{ij}^k\}\|_2^2\}
\nonumber\\
&\textstyle\overset{(a)}{=}
\textstyle\textstyle\sum_{j=1}^{P_i}\mathbb{E}_{\{t_{ij}^{k}\}}^{\mathcal{F}^{k-1}}
\Big\{\Big\|S_{ij}\cdot\Big(\nabla f_{ij,t_{ij}^{k}}(\boldsymbol{x}_i^{k})-
\nabla f_{ij,t_{ij}^{k}}(\boldsymbol{\phi}_{ij,t_{ij}^{k}}^{k-1})
\nonumber\\
&\textstyle\qquad\qquad\quad-\mathbb{E}_{t_{ij}^k}^{\mathcal{F}^{k-1}}
\big\{\nabla f_{ij,t_{ij}^k}(\boldsymbol{x}_i^{k})-
\nabla f_{ij,t_{ij}^k}(\boldsymbol{\phi}_{ij,t_{ij}^k}^{k-1})\big\}\Big)\Big\|_2^2\Big\}
\nonumber\\
&\overset{(\ref{def-VD})}{\leq}\textstyle
\textstyle\sum_{j=1}^{P_i}S_{ij}^2\cdot\mathbb{E}_{t_{ij}^{k}}^{\mathcal{F}^{k-1}}
\Big\{\|\nabla f_{ij,t_{ij}^k}(\boldsymbol{x}_i^{k})-\nabla f_{ij,t_{ij}^{k}}
(\boldsymbol{\phi}_{ij,t_{ij}^{k}}^{k-1})\|_2^2\Big\}
\nonumber\\
&= \textstyle \textstyle\sum_{j=1}^{P_i}S_{ij}\sum_{t=1}^{S_{ij}}
\|\nabla f_{ij,t}(\boldsymbol{x}_i^{k})-\nabla f_{ij,t}
(\boldsymbol{\phi}_{ij,t}^{k-1})\|_2^2
\nonumber\\
&\overset{(b)}{\leq}\textstyle\textstyle\sum\limits_{j=1}^{P_i}
2L^2S_{ij}\sum\limits_{t=1}^{S_{ij}}
\Big(\|\boldsymbol{x}_i^{k}-\boldsymbol{x}^{*}\|_2^2+\|\boldsymbol{x}^*
-\boldsymbol{\phi}_{ij,t}^{k-1}\|_2^2\Big),
\label{lemmaX-6}
\end{align}
where $(a)$ used the definition of $\boldsymbol{g}_{ij}^k$ and
the fact that $\mathbb{E}_{t_{ij}^{k}}^{\mathcal{F}^{k-1}}
\{\boldsymbol{g}_{ij}^k\}=\nabla f_{ij}(\boldsymbol{x}_i^k)=
S_{ij}\cdot\mathbb{E}_{t_{ij}^k}^{\mathcal{F}^{k-1}}
\{\nabla f_{ij,t_{ij}^k}(\boldsymbol{x}_i^{k})\}$, and
$(b)$ used the Lipschitz continuity of $\nabla f_{ij,t}$ as well as
(\ref{def-multi-vect-inequal}). Substituting (\ref{lemmaX-5}) and
(\ref{lemmaX-6}) into (\ref{lemmaX-3}) yields
\begin{align}
&\textstyle\mathbb{E}_{r^{k}}^{\mathcal{F}^{k-1}}
\{G^k\}=\mathbb{E}_{r^{k}}^{\mathcal{F}^{k-1}}
\{\sum_{i=1}^{N}G_i^k\}
\nonumber\\
&\leq\textstyle \sum_{i=1}^N \big(4L^2(1+P_iS_{ij}^2)
\|\boldsymbol{x}_i^k-\boldsymbol{x}^*\|_2^2+
\nonumber\\
&\textstyle \sum\limits_{j=1}^{P_i}4L^2S_{ij}
\sum\limits_{t=1}^{S_{ij}}\|\boldsymbol{x}^*
-\boldsymbol{\phi}_{ij,t}^{k-1}\|_2^2+6|\mathcal{N}_{i}^{k}|^2\rho
\|\boldsymbol{x}_i^{k}-[\bar{\boldsymbol{W}}]_{i}\boldsymbol{x}^{k}\|_2^2\big)
\nonumber\\
&\leq\textstyle \sum_{i=1}^N 8L^2(1+P_iS_{ij}^2)
(\|\bar{\boldsymbol{x}}^k-\boldsymbol{x}^*\|_2^2+
\|\boldsymbol{x}_i^k-\bar{\boldsymbol{x}}^k\|_2^2)
\nonumber\\
&\textstyle \qquad+4L^2\overline{S_{ij}}\cdot D^{k-1}+
6\overline{P_i}^2\rho\|\boldsymbol{x}^{k}-
\bar{\boldsymbol{W}}\boldsymbol{x}^{k}\|_2^2,
\nonumber\\
&\overset{(a)}{\leq}\textstyle
c_2\bar{X}^k+c_1 X^k+c_3D^{k-1},
\label{lemmaX-7}
\end{align}
where $\overline{S_{ij}}=\max_{i,j}\{S_{ij}\}$,
$\overline{P_i}\triangleq\max_i\{P_i\}$, $\{c_i\}_{i=1}^3$ are
defined in (\ref{lemma4-1-content2}), and $(a)$ is because
\begin{align}
&\textstyle\|\boldsymbol{x}^{k}
-\bar{\boldsymbol{W}}\boldsymbol{x}^{k}\|_2^2
\leq2\|\boldsymbol{x}^{k}-\bar{\boldsymbol{W}}_{\infty}\boldsymbol{x}^{k}\|_2^2
+2\|\bar{\boldsymbol{W}}\boldsymbol{x}^{k}-
\bar{\boldsymbol{W}}_{\infty}\boldsymbol{x}^{k}\|_2^2
\nonumber\\
&\textstyle\overset{\text{Lemma} \ \ref{lemma2}}{\leq}
2(1+\sigma^2)\|\boldsymbol{x}^{k}-
\bar{\boldsymbol{W}}_{\infty}\boldsymbol{x}^{k}\|_2^2
=2(1+\sigma^2)X^k.
\label{lemmaX-14}
\end{align}
Taking a full expectation for both sides of (\ref{lemmaX-7})
yields (\ref{lemma4-1-content}).
\end{proof}

With the help of Lemma \ref{lemma4-1}, we can proceed to prove
(\ref{lemma4-1-1-content}).

\newtheorem{lemma8}{Lemma}
\begin{lemma}
\label{lemma5}
Suppose the assumptions in Section \ref{sec-obj-assumption} hold.
Let $(\boldsymbol{x}^{k+1},\boldsymbol{y}^{k+1})$ be generated by
the proposed Algorithm \ref{alg-4}. It is also assumed that
$\alpha\leq\frac{1}{L}$ and $\boldsymbol{y}^0=\boldsymbol{g}^0=
\boldsymbol{0}$. Then the following inequalities hold:
\begin{align}
&\textstyle\mathbb{E}\{\bar{X}^{k+1}\}
\leq \mathbb{E}\{b_1\cdot\bar{X}^{k}+b_2\cdot X^k+
b_3\cdot D^{k-1}\}
\label{lemma5-1-content1}\\
&\textstyle\mathbb{E}\{\bar{X}^{k+1}\}
\leq\mathbb{E}\{ (2+\frac{2c_1}{L^2})\cdot\bar{X}^{k}
+\bar{b}\cdot X^k+b_3\cdot D^{k-1}\}
\label{lemma5-1-conteng2}
\end{align}
where
\begin{align}
&\textstyle b_1=1-\mu\alpha+\frac{2c_2\alpha^2}{N},
b_2=\frac{2\alpha L^2+4\alpha \rho(1+\sigma^2)\overline{P_i}^2+
2\mu\alpha^2L^2+2\mu\alpha^2c_1}{\mu N},
\nonumber\\
&\textstyle b_3=\frac{2\alpha^2c_3}{N}, \
\bar{b}=\frac{4\alpha^2L^2+4\alpha^2\rho(1+\sigma^2)
\overline{P_i}^2+2\alpha^2c_1}{N}.
\label{lemma5-1-conteng3}
\end{align}
\end{lemma}
\newtheorem{comment}{Comment}
\begin{comment}
Note that (\ref{lemma5-1-content1}) is the second inequality in (\ref{proofn-2}).
While the inequality (\ref{lemma5-1-conteng2}) will be used in Lemma \ref{lemma9}.
\end{comment}
\begin{proof}
Define $\bar{\boldsymbol{g}}^{k}\triangleq\frac{1}{N} (\boldsymbol{1}_N^T\otimes
\boldsymbol{I}_d)\boldsymbol{g}^k$ as the average of $\boldsymbol{g}^k_i$s.
Since $\boldsymbol{y}^0=\boldsymbol{g}^0=\boldsymbol{0}$, according
to (\ref{PropAlg-compact-form}) it holds that
\begin{align}
&\textstyle\boldsymbol{y}^{1}\overset{(\ref{PropAlg-compact-form})}{=}
\bar{\boldsymbol{W}}\boldsymbol{y}^{0}+\boldsymbol{g}^{1}-\boldsymbol{g}^0
\Rightarrow
\boldsymbol{y}^{1}=\boldsymbol{g}^{1}
\nonumber\\
&\textstyle
\Rightarrow
\bar{\boldsymbol{W}}_{\infty}\boldsymbol{y}^{2}
\overset{(\ref{PropAlg-compact-form})}{=}
\bar{\boldsymbol{W}}_{\infty}(\bar{\boldsymbol{W}}
\boldsymbol{y}^{1}+\boldsymbol{g}^{2}-\boldsymbol{y}^1)
\overset{(a)}{\Rightarrow}
\bar{\boldsymbol{y}}^{2}=\bar{\boldsymbol{g}}^{2}
\label{lemma5-2-1}
\end{align}
where $(a)$ is because $\bar{\boldsymbol{W}}_{\infty}\boldsymbol{y}^{2}
=\bar{\boldsymbol{y}}^{2}$, $\bar{\boldsymbol{W}}_{\infty}
\bar{\boldsymbol{W}}=\bar{\boldsymbol{W}}_{\infty}$ and
$\bar{\boldsymbol{W}}_{\infty}\boldsymbol{g}^{2}=\bar{\boldsymbol{g}}^{2}$.
Using deduction we have
\begin{align}
\bar{\boldsymbol{g}}^{k}=\bar{\boldsymbol{y}}^{k}, \ \forall k.
\label{lemma5-2-2}
\end{align}
Multiplying $\frac{\boldsymbol{1}_N^T\otimes \boldsymbol{I}_d}{N}$
to the first line of (\ref{PropAlg-compact-form}) yields
\begin{align}
\bar{\boldsymbol{x}}^{k+1}\overset{(a)}{=}\bar{\boldsymbol{x}}^{k}-
\alpha\bar{\boldsymbol{y}}^{k}\overset{(b)}{=}\bar{\boldsymbol{x}}^{k}-
\alpha\bar{\boldsymbol{g}}^{k}
\label{appendix-a-1}
\end{align}
where $(a)$ is because $\frac{\boldsymbol{1}_N^T\otimes
\boldsymbol{I}_d}{N}\bar{\boldsymbol{W}}=\frac{\boldsymbol{1}_N^T
\otimes\boldsymbol{I}_d}{N}$ ($\boldsymbol{W}$ is doubly
stochastic), and $(b)$ is because $\bar{\boldsymbol{g}}^{k}=
\bar{\boldsymbol{y}}^{k}$. Then we have
\begin{align}
&\textstyle\mathbb{E}_{r^k}^{\mathcal{F}^{k-1}}
\{\bar{X}^{k+1}\}
\overset{(\ref{appendix-a-1})}{=}
\mathbb{E}_{r^k}^{\mathcal{F}^{k-1}}
\{\|\bar{\boldsymbol{x}}^{k}-\alpha\bar{\boldsymbol{g}}^{k}
-\boldsymbol{x}^*\|_2^2\}
\nonumber\\
&=\textstyle\mathbb{E}_{r^k}^{\mathcal{F}^{k-1}}
\{\|\bar{\boldsymbol{x}}^{k}-\alpha\nabla f(\bar{\boldsymbol{x}}^k)+
\alpha\nabla f(\bar{\boldsymbol{x}}^k)-\alpha\bar{\boldsymbol{g}}^{k}-
\boldsymbol{x}^*\|_2^2\}
\nonumber\\
&=\textstyle \|\bar{\boldsymbol{x}}^{k}-\boldsymbol{x}^*-
\alpha\nabla f(\bar{\boldsymbol{x}}^k)\|_2^2+
\mathbb{E}_{r^k}^{\mathcal{F}^{k-1}}
\Big\{\alpha^2\underbrace{\|\nabla f(\bar{\boldsymbol{x}}^k)
-\bar{\boldsymbol{g}}^{k}\|_2^2}_{\text{[(\ref{appendix-a-2})-2]}}
\nonumber\\
&\textstyle+2\alpha\underbrace{\langle\bar{\boldsymbol{x}}^{k}-
\boldsymbol{x}^*-\alpha\nabla f(\bar{\boldsymbol{x}}^k),
\nabla f(\bar{\boldsymbol{x}}^k)
-\bar{\boldsymbol{g}}^{k}\rangle}_{\text{[(\ref{appendix-a-2})-1]}}\Big\}
\nonumber\\
&\overset{(a)}{\leq}\textstyle (1-\mu\alpha)^2\bar{X}^k
+\mathbb{E}_{r^k}^{\mathcal{F}^{k-1}}
\{2\alpha\text{[(\ref{appendix-a-2})-1]}
+\alpha^2\text{[(\ref{appendix-a-2})-2]}\}
\label{appendix-a-2}
\end{align}
where $(a)$ invoked Lemma \ref{lemma1}.

\subsubsection{Bounding $\text{[(\ref{appendix-a-2})-1]}$ and $\text{[(\ref{appendix-a-2})-2]}$}
Regarding the term $\text{[(\ref{appendix-a-2})-1]}$, we have
\begin{align}
&\mathbb{E}_{r^k}^{\mathcal{F}^{k-1}}\{\text{[(\ref{appendix-a-2})-1]}\}
\nonumber\\
&\textstyle
=\mathbb{E}_{r^k}^{\mathcal{F}^{k-1}}
\{\langle\underbrace{\bar{\boldsymbol{x}}^{k}-\boldsymbol{x}^*-
\alpha\nabla f(\bar{\boldsymbol{x}}^k)}_{\text{[(\ref{appendix-a-4})-1]}},
\nabla f(\bar{\boldsymbol{x}}^k)-\bar{\boldsymbol{g}}^{k}\rangle\}
\nonumber\\
&=\textstyle  \mathbb{E}_{r^k}^{\mathcal{F}^{k-1}}
\Big\{\langle\text{[(\ref{appendix-a-4})-1]},\nabla f(\bar{\boldsymbol{x}}^k)-
\bar{\nabla} f(\boldsymbol{x}^k)\rangle+
\nonumber\\
&\textstyle
\qquad\langle\text{[(\ref{appendix-a-4})-1]},\bar{\nabla} f(\boldsymbol{x}^k)
-(\bar{\boldsymbol{g}}^{k}+\frac{1}{N}\sum_{i=1}^N
\sum_{j\in\mathcal{N}_i^{k}}\boldsymbol{\Delta}_{ij}^{k})\rangle
\nonumber\\
&\qquad\qquad+\textstyle
\underbrace{\textstyle\langle\text{[(\ref{appendix-a-4})-1]},
\frac{1}{N}\sum_{i=1}^N\sum_{j\in\mathcal{N}_i^{k}}
\boldsymbol{\Delta}_{ij}^{k}\rangle}_{\text{[(\ref{appendix-a-4})-2]}}
\Big\}
\nonumber\\
&\overset{(a)}{\leq}\textstyle \mathbb{E}_{r^k}^{\mathcal{F}^{k-1}}
\{\|\text{[(\ref{appendix-a-4})-1]}\|_2\cdot\|\nabla f(\bar{\boldsymbol{x}}^k)-
\bar{\nabla} f(\boldsymbol{x}^k)\|_2+\text{[(\ref{appendix-a-4})-2]}\}
\nonumber\\
&\overset{(b)}{\leq}\textstyle \mathbb{E}_{r^k}^{\mathcal{F}^{k-1}}
\{(1-\mu\alpha)\bar{X}^k\cdot
\|\nabla f(\bar{\boldsymbol{x}}^k)-\bar{\nabla} f(\boldsymbol{x}^k)\|_2
+\text{[(\ref{appendix-a-4})-2]}\}
\nonumber\\
&\overset{(c)}{\leq}\textstyle \mathbb{E}_{r^k}^{\mathcal{F}^{k-1}}
\Big\{\frac{p(1-\mu\alpha)}{2}\bar{X}^k+
\frac{(1-\mu\alpha)}{2p}\|\nabla f(\bar{\boldsymbol{x}}^k)-
\bar{\nabla} f(\boldsymbol{x}^k)\|_2^2
\nonumber\\
&\textstyle\qquad\qquad+\text{[(\ref{appendix-a-4})-2]}\Big\}
\nonumber\\
&\overset{(d)}{\leq}\textstyle\mathbb{E}_{r^k}^{\mathcal{F}^{k-1}}
\big\{\frac{p(1-\mu\alpha)}{2}\bar{X}^k+\frac{L^2}{2pN}X^k
+\text{[(\ref{appendix-a-4})-2]}\big\}, \forall p>0,
\label{appendix-a-4}
\end{align}
where $\bar{\nabla} f(\boldsymbol{x})
=\frac{1}{N}(\boldsymbol{1}_N^T\otimes\boldsymbol{I}_d) \nabla
f(\boldsymbol{x})$ is the average of $\nabla
f_i(\boldsymbol{x}_i)$s, and $(a)$ has invoked the Cauchy-Schwarz
inequality as well as the fact that
\begin{align}
\textstyle\mathbb{E}_{r^k}^{\mathcal{F}^{k-1}}
\{\bar{\nabla} f(\boldsymbol{x}^k)-(\bar{\boldsymbol{g}}^{k}
+\frac{1}{N}\sum_{i=1}^N\sum_{j\in\mathcal{N}_i^{k}}
\boldsymbol{\Delta}_{ij}^{k})\}=\boldsymbol{0},
\nonumber
\end{align}
$(b)$ used Lemma \ref{lemma1}, $(c)$ uses the inequality
$\langle\boldsymbol{x},\boldsymbol{y}\rangle\leq \frac{p}{2}
\|\boldsymbol{x}\|_2^2+\frac{1}{2p}\|\boldsymbol{y}\|_2^2$, $\forall p>0$,
and $(d)$ invoked Lemma \ref{lemma3} as well as $1-\mu\alpha<1$.
For [(\ref{appendix-a-4})-2] we have
\begin{align}
&\textstyle
\mathbb{E}_{r^k}^{\mathcal{F}^{k-1}}\{\text{[(\ref{appendix-a-4})-2]}\}
\overset{(a)}{\leq}
\mathbb{E}_{r^k}^{\mathcal{F}^{k-1}}\Big\{
(1-\mu\alpha)\bar{X}^k\|\frac{1}{N}\sum\limits_{i=1}^N
\sum\limits_{j\in\mathcal{N}_i^{k}}\boldsymbol{\Delta}_{ij}^{k}\|_2\Big\}
\nonumber\\
&\textstyle\leq
\frac{p(1-\mu\alpha)}{2}\bar{X}^k+
\mathbb{E}_{r^k}^{\mathcal{F}^{k-1}}\Big\{\frac{(1-\mu\alpha)}{2p}
\|\frac{1}{N}\sum_{i=1}^N\sum_{j\in\mathcal{N}_i^{k}}
\boldsymbol{\Delta}_{ij}^{k}\|_2^2\Big\}
\nonumber\\
&\overset{(b)}{\leq} \textstyle \frac{p(1-\mu\alpha)}{2}\bar{X}^k+
\frac{1}{2pN}\sum_{i=1}^NP_i\sum_{j=1}^{P_i}\rho\|\boldsymbol{x}_i^{k}-
[\bar{\boldsymbol{W}}]_{i}\boldsymbol{x}^{k}\|_2^2
\nonumber\\
&\textstyle\overset{(\ref{lemmaX-14})}{\leq} \textstyle \frac{p(1-\mu\alpha)}{2}
\bar{X}^k+\frac{\rho(1+\sigma^2)\overline{P_i}^2}{pN}X^k, \ \forall p>0,
\label{appendix-a-4-2}
\end{align}
where $(a)$ has invoked the Cauchy-Schwarz inequality as well as
Lemma \ref{lemma1}, and $(b)$ used (\ref{def-multi-vect-inequal})
and (\ref{lemmaX-2}). As for $\text{[(\ref{appendix-a-2})-2]}$, we
have
\begin{align}
&\textstyle\mathbb{E}_{r^k}^{\mathcal{F}^{k-1}}
\{\text{[(\ref{appendix-a-2})-2]}\}
\nonumber\\
&\textstyle
\overset{(\ref{def-triangle-inequal})}{\leq}
\textstyle2\|\nabla f(\bar{\boldsymbol{x}}^k)-
\bar{\nabla} f(\boldsymbol{x}^k)\|_2^2
+\mathbb{E}_{r^k}^{\mathcal{F}^{k-1}}
\{2\|\bar{\nabla} f(\boldsymbol{x}^k)-\bar{\boldsymbol{g}}^{k}\|_2^2\}
\nonumber\\
&\overset{(a)}{\leq}\textstyle\frac{2L^2}{N}X^k+
\mathbb{E}_{r^k}^{\mathcal{F}^{k-1}}\{\frac{2}{N}
\sum_{i=1}^N\|\boldsymbol{g}_i^k-\nabla f_i(\boldsymbol{x}^k)\|_2^2\}
\nonumber\\
&=\textstyle\frac{2L^2}{N}X^k+
\mathbb{E}_{r^k}^{\mathcal{F}^{k-1}}\{\frac{2}{N}
\|\boldsymbol{g}^k-\nabla f(\boldsymbol{x}^k)\|_2^2\}
\label{appendix-a-5}
\end{align}
where $(a)$ comes from the definitions of $\bar{\nabla} f(\boldsymbol{x}^k)$
and $\bar{\boldsymbol{g}}^{k}$, Lemma \ref{lemma3} as well as
(\ref{def-multi-vect-inequal}).

\subsubsection{Combining}
Substituting (\ref{appendix-a-4}), (\ref{appendix-a-4-2}) and
(\ref{appendix-a-5}) into (\ref{appendix-a-2}) yields
\begin{align}
&\textstyle\mathbb{E}_{r^k}^{\mathcal{F}^{k-1}}
\{\|\bar{\boldsymbol{x}}^{k+1}-\boldsymbol{x}^*\|_2^2\}
\leq \underbrace{\textstyle\big((1-\mu\alpha)^2+
2\alpha p(1-\mu\alpha)\big)\bar{X}^k}_{\text{[(\ref{appendix-a-7})-1]}}
\nonumber\\
&\textstyle
+\underbrace{\textstyle\big(\frac{\alpha L^2}{pN}+
\frac{2\alpha\rho(1+\sigma^2)\overline{P_i}^2}{pN}+
\frac{2\alpha^2L^2}{N}\big)X^k}_{\text{[(\ref{appendix-a-7})-2]}}
\nonumber\\
&\textstyle+\mathbb{E}_{r^k}^{\mathcal{F}^{k-1}}
\{\frac{2\alpha^2}{N}\|\boldsymbol{g}^k-\nabla f(\boldsymbol{x}^k)\|_2^2\}
\nonumber\\
&\overset{(\ref{lemmaX-7})}{\leq}\textstyle \text{[(\ref{appendix-a-7})-1]}
+\text{[(\ref{appendix-a-7})-2]}+
\frac{2\alpha^2}{N}(c_1X^k+c_2\bar{X}^k+c_3D^{k-1})
\label{appendix-a-7}
\end{align}
In (\ref{appendix-a-7}), respectively setting $p=\frac{\mu}{2}$ and
$p=\frac{1}{2\alpha}$ yields
\begin{align}
&\textstyle\mathbb{E}_{r^k}^{\mathcal{F}^{k-1}}
\{\bar{X}^{k+1}\}
 \leq b_1\cdot\bar{X}^{k}+b_2\cdot X^k+b_3\cdot D^{k-1},
\label{appendix-a-8}\\
&\textstyle\mathbb{E}_{r^k}^{\mathcal{F}^{k-1}}\{\bar{X}^{k+1}\}
\overset{(a)}{\leq}\textstyle (2+\frac{2c_2}{L^2})\bar{X}^k
+\bar{b} X^k+b_3 D^{k-1}
\label{appendix-a-9}
\end{align}
where in $(a)$ we used the assumption that $\alpha\leq\frac{1}{L}$,
and $b_2$, $b_3$ and $\bar{b}$ are defined in (\ref{lemma5-1-conteng3}).
Taking the full expectation for both sides of (\ref{appendix-a-8})
and (\ref{appendix-a-9}) yields the desired results.
\end{proof}

\section{Proving the Third Inequality in (\ref{proofn-2})}
\label{appendix-C}
The third inequality in (\ref{proofn-2}), i.e., (\ref{lemma5-1-content}),
is proved in the following lemma.
\newtheorem{lemma8-1}{Lemma}
\begin{lemma}
\label{lemma5-1}
Suppose the assumptions in Section \ref{sec-obj-assumption} hold. Also
let $(\boldsymbol{x}^{k+1}, \boldsymbol{y}^{k+1})$ be generated by
Algorithm \ref{alg-4}. Then
\begin{align}
&\textstyle\mathbb{E}\{D^{k}\}\leq
\mathbb{E}\big\{ (1-\frac{1}{\overline{S_{ij}}}) D^{k-1}
+ 2\overline{P_i}X^k+2\overline{P_i}N\cdot\bar{X}^k\big\}
\label{lemma5-1-content}
\end{align}
\end{lemma}
\begin{proof}
For $D^{k}$ we have
\begin{align}
&\textstyle\mathbb{E}_{r^k}^{\mathcal{F}^{k-1}}\{D^{k}\}
=\mathbb{E}_{\{t_{ij}^{k}\}}^{\mathcal{F}^{k-1}}
\{\sum_{i=1}^{N}\sum_{j=1}^{P_i}\sum_{t_{ij}^{k}=1}^{S_{ij}}
\|\boldsymbol{x}^*-\boldsymbol{\phi}_{ij,t_{ij}^{k}}^{k}\|_2^2\}
\nonumber\\
&\textstyle= \sum\limits_{i=1}^{N}\sum\limits_{j=1}^{P_i}\sum\limits_{t=1}^{S_{ij}}
\big((1-\frac{1}{S_{ij}})\|\boldsymbol{x}^*-\boldsymbol{\phi}_{ij,t}^{k-1}\|_2^2
+\frac{1}{S_{ij}}\|\boldsymbol{x}^*-\boldsymbol{x}_{i}^{k}\|_2^2\big)
\nonumber\\
&\textstyle\leq (1-\frac{1}{\overline{S_{ij}}}) D^{k-1}
+ 2\overline{P_i}\sum_{i=1}^N (\|\boldsymbol{x}_i^{k}-
\bar{\boldsymbol{x}}^{k}\|_2^2+\|\bar{\boldsymbol{x}}^{k}-\boldsymbol{x}^*\|_2^2)
\nonumber\\
&\textstyle= (1-\frac{1}{\overline{S_{ij}}}) D^{k-1}
+ 2\overline{P_i}X^k+2\overline{P_i}N\cdot \bar{X}^k
\label{lemmaY-1}
\end{align}
where $\overline{S_{ij}}=\max_{i,j}\{S_{ij}\}$ and $\overline{P_i}=
\max_{i}\{P_{i}\}$.
\end{proof}

\section{Proving the Forth Inequality in (\ref{proofn-2})}
\label{appendix-D}
The forth inequality in (\ref{proofn-2}), i.e., (\ref{lemma9-content1}),
is proved in the following lemma.
\newtheorem{lemma7}{Lemma}
\begin{lemma}
\label{lemma9}
Suppose the assumptions in Section \ref{sec-obj-assumption} hold.
Let $(\boldsymbol{x}^{k+1},\boldsymbol{y}^{k+1})$ be generated
by Algorithm \ref{alg-4}. Also assume that
$\alpha<\frac{1}{4\sqrt{2}L}$. Then
\begin{align}
\textstyle\mathbb{E}\{Y^{k+1}\}
\leq\mathbb{E}\big\{a_1 X^k+a_2 \bar{X}^k+
a_3 Y^k+a_4 D^{k-1}\big\}
\label{lemma9-content1}
\end{align}
where
\begin{align}
&\textstyle a_1=\frac{25L^2(1+\sigma^2)+4(1+\sigma^2)c_1
+3(1+\sigma^2)(2c_1\sigma^2+c_2\bar{b}+
2c_3\overline{P_i}N)}{1-\sigma^2},
\nonumber\\
&\textstyle
a_2=\frac{NL^2(1+\sigma^2)+4(1+\sigma^2)c_2+
3(1+\sigma^2)(c_2(2+\frac{2c_2}{L^2})+
2c_3\overline{P_i}N)}{1-\sigma^2},
\nonumber\\
&\textstyle a_3=\frac{1+\sigma^2}{2}+
\frac{24\alpha^2 L^2(1+\sigma^2)}{1-\sigma^2}
+\frac{6\alpha^2c_1(1+\sigma^2)}{1-\sigma^2},
\nonumber\\
&\textstyle
a_4=\frac{4(1+\sigma^2)c_3}{1-\sigma^2}
+\frac{3(1+\sigma^2)(b_3c_2+c_3(1-\overline{S_{ij}}^{-1}))}{1-\sigma^2}
\label{lemma9-content2}
\end{align}
\end{lemma}
\begin{proof}
The proof of this lemma is split into three parts. The first part provides
an upper bound of $\mathbb{E}\{Y^{k+1}\}$. However, the term
$G^{k+1}\triangleq\|\boldsymbol{g}^{k+1}-\nabla f(\boldsymbol{x}^{k+1})\|_2^2$
contained in this bound needs to be further bounded. After
bounding $G^{k+1}$ in the second part, the desired result is proved in
the third part.

\emph{Part I}: According to the second line of (\ref{PropAlg-compact-form}),
we have
\begin{align}
&\textstyle Y^{k+1}\overset{(a)}{=} 
\|(\bar{\boldsymbol{W}}-
\bar{\boldsymbol{W}}_{\infty})\boldsymbol{y}^{k}
+(\boldsymbol{I}-\bar{\boldsymbol{W}}_{\infty})(\boldsymbol{g}^{k+1}
-\boldsymbol{g}^k))\|_2^2
\nonumber\\
&\overset{(\ref{def-triangle-inequal})}
{\leq}\textstyle(1+p)\|(\bar{\boldsymbol{W}}-
\bar{\boldsymbol{W}}_{\infty})\boldsymbol{y}^{k}\|_2^2
\nonumber\\
&\textstyle+(1+\frac{1}{p})\|(\boldsymbol{I}-\bar{\boldsymbol{W}}_{\infty})
(\boldsymbol{g}^{k+1}-\boldsymbol{g}^k))\|_2^2
\nonumber\\
&\overset{(b)}{\leq}\textstyle(1+p)\sigma^2 Y^k
+(1+p^{-1})\|\boldsymbol{g}^{k+1}-\boldsymbol{g}^k\|_2^2
\nonumber\\
&\overset{(c)}{=}
\textstyle\frac{1+\sigma^2}{2} Y^k
+\frac{1+\sigma^2}{1-\sigma^2}\|\boldsymbol{g}^{k+1}-\boldsymbol{g}^k+
\nabla f(\boldsymbol{x}^{k+1})-\nabla f(\boldsymbol{x}^{k+1})
\nonumber\\
&\textstyle\qquad\qquad\qquad\qquad\qquad-\nabla f(\boldsymbol{x}^{k})+
\nabla f(\boldsymbol{x}^{k})\|_2^2
\nonumber\\
&\overset{(d)}{\leq}\textstyle\frac{1+\sigma^2}{2} Y^k
+\frac{3(1+\sigma^2)}{1-\sigma^2}(G^{k+1}+G^k
+L^2\|\boldsymbol{x}^{k+1}-\boldsymbol{x}^{k}\|_2^2)
\label{lemma4-4}
\end{align}
where $(a)$ is because of $\bar{\boldsymbol{W}}_{\infty}
\bar{\boldsymbol{W}}=\bar{\boldsymbol{W}}_{\infty}$, $(b)$ is
because of Lemma \ref{lemma2} as well as $\|\boldsymbol{I}-
\bar{\boldsymbol{W}}_{\infty}\|_2=1$, $(c)$ is obtained by
setting $p=\frac{1-\sigma^2}{2\sigma^2}$, and $(d)$ used
(\ref{def-multi-vect-inequal}) as well as the Lipschitz continuity
of $\nabla f$. Regarding $\|\boldsymbol{x}^{k+1}-\boldsymbol{x}^{k}\|_2^2$,
we have
\begin{align}
&\|\boldsymbol{x}^{k+1}-\boldsymbol{x}^{k}\|_2^2
\overset{(a)}{=}\|(\bar{\boldsymbol{W}}-\boldsymbol{I})
(\boldsymbol{x}^{k}-\bar{\boldsymbol{W}}_{\infty}\boldsymbol{x}^{k})-
\alpha\boldsymbol{y}^k\|_2^2
\nonumber\\
&\overset{(b)}{\leq} 8X^k
+2\alpha^2\|\boldsymbol{y}^k\|_2^2
\nonumber\\
&\textstyle\overset{(c)}{=}8X^k
+2\alpha^2\|\boldsymbol{y}^k-\bar{\boldsymbol{W}}_{\infty}(\boldsymbol{y}^k
-\boldsymbol{g}^k+\nabla f(\boldsymbol{x}^k)-\nabla f(\boldsymbol{x}^k)
\nonumber\\
&\qquad\qquad\qquad+\nabla f(\boldsymbol{x}^*))\|_2^2
\nonumber\\
&\textstyle\overset{(d)}{\leq}8X^k+2\alpha^2(\sqrt{Y^k}
+\sqrt{G^k}+
L\|\boldsymbol{x}^k-(\boldsymbol{1}_N\otimes\boldsymbol{I}_d)
\boldsymbol{x}^*\|_2)^2
\nonumber\\
&\textstyle\overset{(e)}{\leq}8X^k+2\alpha^2(\sqrt{Y^k}
+\sqrt{G^k}+L\sqrt{X^k}+L\sqrt{N}\sqrt{\bar{X}^k})^2
\nonumber\\
&\textstyle\overset{(\ref{def-multi-vect-inequal})}{\leq}8X^k+
8\alpha^2(Y^k+G^k+L^2X^k+L^2 N \bar{X}^k)
\label{lemma4-6}
\end{align}
where $(a)$ is because of (\ref{PropAlg-compact-form}) and
$\bar{\boldsymbol{W}}_{\infty}\bar{\boldsymbol{W}}=
\bar{\boldsymbol{W}}_{\infty}$, $(b)$ used
(\ref{def-multi-vect-inequal}) and
$\|\bar{\boldsymbol{W}}-\boldsymbol{I}\|_2\leq 2$ (using triangle
inequality with $\|\bar{\boldsymbol{W}}\|_2<1$), $(c)$ is because
of
\begin{align}
\bar{\boldsymbol{W}}_{\infty}\boldsymbol{y}^k=(\boldsymbol{1}_N
\otimes\boldsymbol{I}_d)\bar{\boldsymbol{y}}^{k}
\overset{(\ref{lemma5-2-2})}{=}(\boldsymbol{1}_N
\otimes\boldsymbol{I}_d)\bar{\boldsymbol{g}}^{k}=
\bar{\boldsymbol{W}}_{\infty}\boldsymbol{g}^k
\label{lemma4-7}
\end{align}
as well as $\bar{\boldsymbol{W}}_{\infty}\nabla f(\boldsymbol{x}^*)
=\boldsymbol{0}$, $(d)$ used the triangle inequality and the fact
that $\|\bar{\boldsymbol{W}}_{\infty}\|_2=1$ as well as the Lipschitz
continuity of $\nabla f$, $(e)$ is due to
\begin{align}
&\|\boldsymbol{x}^k-(\boldsymbol{1}_N\otimes\boldsymbol{I}_d)\boldsymbol{x}^*\|_2
\leq
\|\boldsymbol{x}^k-(\boldsymbol{1}_N\otimes\boldsymbol{I}_d)\bar{\boldsymbol{x}}^k\|_2
\nonumber\\
&+\|(\boldsymbol{1}_N\otimes\boldsymbol{I}_d)(\bar{\boldsymbol{x}}^k-
\boldsymbol{x}^*)\|_2
\leq \sqrt{X^k}+\sqrt{N}\sqrt{\bar{X}^k}.
\end{align}
Substituting (\ref{lemma4-6}) into (\ref{lemma4-4}) and also using the
assumption $24L^2\alpha^2\leq 1$ we obtain
\begin{align}
&\textstyle Y^{k+1}
\leq (1+\sigma^2)\big((\frac{1}{2}+
\frac{24\alpha^2 L^2}{1-\sigma^2})Y^k+\frac{25L^2}{1-\sigma^2}X^k
+\frac{NL^2}{1-\sigma^2}\bar{X}^k+
\nonumber\\
&\textstyle\frac{4}{1-\sigma^2}G^k+\frac{3}{1-\sigma^2}G^{k+1}\big).
\label{lemma4-8}
\end{align}

\emph{Part II}: In (\ref{lemma4-8}), the term $G^{k+1}$ can be
bounded by
\begin{align}
&\textstyle\mathbb{E}\{G^{k+1}\}
\overset{(\ref{lemma4-1-content})}{\leq}
\mathbb{E}\{c_1X^{k+1}+c_2\bar{X}^{k+1}+c_3 D^{k}\}
\nonumber\\
&\overset{(a)}{\leq}\textstyle \mathbb{E}\Big\{ \big(2c_1\sigma^2+c_2\bar{b}+
2c_3\overline{P_i}) X^k+
(c_2(2+\frac{2c_2}{L^2})+2c_3\overline{P_i}N)\bar{X}^k
\nonumber\\
&\textstyle \qquad+(b_3c_2+c_3(1-\frac{1}{\overline{S_{ij}}}))
\cdot D^{k-1}+2\alpha^2c_1\cdot Y^k\Big\}
\label{lemmaZ-1}
\end{align}
where $(a)$ is obtained by invoking (\ref{lemma10-content2}),
(\ref{lemma5-1-conteng2}) and (\ref{lemma5-1-content}).

\emph{Part III}: Substituting (\ref{lemma4-1-content}) and
(\ref{lemmaZ-1}) into (\ref{lemma4-8}) yields
\begin{align}
&\textstyle\mathbb{E}\{Y^{k+1}\}
=\mathbb{E}\big\{a_1 X^k+a_2 \bar{X}^k+a_3 Y^k+a_4  D^{k-1}\big\}
\label{lemmaB-1}
\end{align}
where $\{a_i\}_{i=1}^4$ are defined in (\ref{lemma9-content2}).
\end{proof}

\end{document}